\documentclass[10pt, a4paper, twocolumn]{IEEEtran}

\usepackage[numbers]{natbib}

\usepackage[USenglish]{babel}

\usepackage{amsmath,amsfonts}
\usepackage{amsthm}  {
       \theoremstyle{plain}
 			\newtheorem{definition}{Definition}
 			\newtheorem{theorem}{Theorem}
 			\newtheorem{assumption}{Assumption}
}

\usepackage{graphicx}
\usepackage{subcaption}
\usepackage{booktabs}
\usepackage{multirow}

\newcommand{\Ti}{\mathcal{T}} 
\newcommand{\St}{\mathcal{S}} 
\newcommand{\X}{\mathcal{X}} 
\newcommand{\M}{\mathcal{M}} 
\newcommand{\U}{\mathcal{U}} 
\newcommand{\Z}{\mathcal{Z}} 
\newcommand{\T}{\mathbb{T}} 
\newcommand{\Ob}{\mathbb{O}} 
\newcommand{\B}{\mathcal{B}} 
\newcommand{\R}{\mathbb{R}} 
\newcommand{\N}{\mathbb{N}} 
\newcommand{\E}{{\mathbb{E}}} 

\DeclareMathOperator*{\argmax}{argmax}
\DeclareMathOperator*{\argsup}{argsup}

\usepackage{algorithm}
\usepackage{algpseudocode}
\usepackage{varwidth}

\usepackage{url}
\usepackage[colorlinks=true,citecolor=red,pdftex]{hyperref} 

\begin{document}

\title{\LARGE \bf Planning for robotic exploration based on forward simulation}

\author{Mikko Lauri, Risto Ritala
\thanks{M. Lauri and R. Ritala are with Department of Automation Science and Engineering, Tampere University of Technology, P.O. Box 692, FI-33101, Tampere, Finland, {\tt\small\{mikko.lauri, risto.ritala\}@tut.fi}.}%
}
\maketitle









\begin{abstract}
We address the problem of controlling a mobile robot to explore a partially known environment.
The robot's objective is the maximization of the amount of information collected about the environment.
We formulate the problem as a partially observable Markov decision process (POMDP) with an information-theoretic objective function, and solve it applying forward simulation algorithms with an open-loop approximation.
We present a new sample-based approximation for mutual information useful in mobile robotics.
The approximation can be seamlessly integrated with forward simulation planning algorithms.
We investigate the usefulness of POMDP based planning for exploration, and to alleviate some of its weaknesses propose a combination with frontier based exploration.
Experimental results in simulated and real environments show that, depending on the environment, applying POMDP based planning for exploration can improve  performance over frontier exploration.
\end{abstract}

\section{Introduction}
\label{sec:intro}
Autonomous robotic agents performing tasks such as monitoring, surveillance or exploration must be able to plan their future information-gathering actions. 
Real-world environments are typically partially observable and stochastic, and planning in them requires reasoning over uncertain outcomes in the presence of sensor noise.
The true state of the system is hidden, and knowledge about the state is represented by a belief state, a probability density function (pdf) over the true state. 
The utility of actions is measured by an appropriate reward function, and the agent's objective is to maximize the sum of expected rewards over a specified horizon of time.
Such planning problems are instances of partially observable Markov decision processes \citep{Kaelbling1998}, or POMDPs.
 
The solution of a POMDP is a control policy, i.e.\ a mapping from belief states to actions.
To find policies for information-gathering and exploration tasks, several authors have proposed applying quantities such as entropy or mutual information as reward functions of the POMDP~\citep{Stachniss2005,Hitz2014,Charrow2014,Atanasov2014,Lauri2015}.
Although finding optimal policies for POMDPs is computationally hard (PSPACE-complete; \cite{Papadimitriou1987}), they remain an attractive modelling choice due to the ability to simultaneously handle uncertainties in the robot's action and sensing outcomes.

In this article, we address the problem of finding control policies for robotic exploration problems formulated as POMDPs.
We apply forward simulation algorithms for finding a solution to an open loop approximation of a POMDP.
This approach allows general belief-dependent reward functions and with suitable choice of algorithm can avoid discretization of the continuous planning space. 
We derive a sampling-based approximation for mutual information that can be applied in conjunction with forward simulation based planning, and describe a method for efficiently drawing the required samples.
We provide an empirical evaluation of our proposed approach in simulated and real-world exploration experiments\footnote{Our software is available at: \url{https://goo.gl/ENGkIf}}.
 
This article is organized as follows. 
Section~\ref{sec:related} provides a survey of related work and discusses the relation of our contribution to the state-of-the-art. 
In Section~\ref{sec:pomdp}, we formulate exploration as a POMDP, and discuss possible solution methods.
Section~\ref{sec:openloop} reviews two forward simulation based methods for non-myopic planning in POMDPs.
Section~\ref{sec:miapprox} introduces the sample-based approximation of mutual information suitable for robotic exploration problems, and describes a method for efficiently drawing the required samples.
Section~\ref{sec:simulation} describes the results of simulation experiments, and Section~\ref{sec:experiment} presents the software architecture for implementation of our approach and reports the results of real-world exploration trials. 
Finally, Section~\ref{sec:conclusion} provides concluding remarks.\section{Related work}
\label{sec:related}
Mobile robots typically collect information on both their internal state and the state of the environment they are interacting with. 
In simultaneous localization and mapping (SLAM) (see \citet{Durrant-Whyte2006} for a review) the robot must jointly estimate both its pose (internal state) and the map of the environment based on its actions and observations. 
Robots' information-gathering actions consist of actions that affect their pose, and hence the area covered by their sensors, and actions explicitly selecting between sensors or their operating modes. 
In the active SLAM problem \citep{Carlone2010}, the robot's actions are selected to obtain best estimates on the pose and the map. 
Thus, the active SLAM problem is an exploration or sensor selection problem. The goal is to maximize information on both the robot pose and the map.

Techniques to control mobile robot exploration may be categorized e.g.\ by whether they apply heuristic rules or formal decision theory for selection of exploration targets.
Applying heuristic rules for guiding exploration spans frontier-based approaches~\citep{Yamauchi1997}, or some next-best-view approaches such as~\cite{Gonzalez-Banos2002}.
\citet{Julia2012} classify exploration methods according to the levels of multi-robot coordination and integration with SLAM algorithms. 
They conclude that SLAM-integrated exploration performs best with regard to quality of map information. 
Their results also agree with \citet{Amigoni2008,Amigoni2010}, who found decision theoretic criteria combining both utility of exploration and its cost (e.g.\ time or distance) to be preferable if both extent of the explored area and exploration time were optimized.
In the following, we review in more detail some single-robot exploration techniques that employ decision theoretic criteria to guide the exploration process.

Information on the location of the robot and environment features, or landmarks, may be modelled by a multivariate Gaussian distribution. 
The SLAM problem can then be solved for example by applying the Extended Kalman Filter. 
Exploration with such feature-based maps was studied by \citet{Sim2005} who describe an A-optimal exploration method, i.e.\ they minimize the trace of the state covariance matrix. 
They discretize the location of the robot to a grid and plan an informative trajectory in open loop as a sequence of discrete positions via a breadth-first search. 
A similar objective function was used by \cite{Kwok2005}, adopting a model predictive control (MPC) approach for optimization over multiple time steps. 
Discretization of the action space was also applied by \cite{Kollar2008}, who applied reinforcement learning to learn parameterized robot trajectories for exploration.
A somewhat complementary approach was adopted in~\cite{Tovar2006}, where a set of candidate exploration targets were evaluated based on an utility function designed to balance exploration of unknown areas and seeing known landmarks to help maintain good localization information.
However, an explicit information-theoretic quantification of the information gain was avoided.

\citet{Martinez-Cantin2009} relaxed assumptions made in earlier work, such as discretization of the action space and the myopic optimization horizon. 
They applied a Monte Carlo search algorithm in policy space with a Gaussian process approximation of the objective function. 
The policies to evaluate were selected by minimizing the average mean square error of the state estimate consisting of robot and landmark locations.

A belief space planning approach investigated by~\cite{Indelman2014a,Indelman2015} addressed many of the limitations of earlier studies. 
A planner architecture consisting of an estimation layer and a decision layer combined with a model predictive control strategy for non-myopic planning was applied, assuming a Gaussian belief over robot and landmark poses. 
Discretization was avoided by applying a gradient descent method for computing optimal actions. Possible future measurements were treated as random variables, relaxing the assumption of maximum likelihood measurements. 
Exploration was considered in the sense that the objective function included an A-optimality criterion for state covariance.

Another body of work in exploration employs metric map representations such as occupancy grids \citep{Moravec1988}. 
\citet{Bourgault2002} combined occupancy grid mapping with feature-based SLAM and used mutual information as the reward function. 
A discretized action space was applied with myopic optimization.

Rao-Blackwellized particle filtering (RBPF) is often applied in state-of-the-art SLAM filters for occupancy grid maps~\citep{Grisetti2007}.
Each particle represents a map and a robot trajectory hypothesis. 
\citet{Stachniss2005} studied myopic exploration in RBPF SLAM by discretizing the action space to a set of possible waypoints and then evaluating the approximate expected information gain when traveling to the waypoints by sampling. 
This work was later expanded upon~\citep{Blanco2008, Carlone2010} by considering alternative measures of uncertainty of the SLAM solution. 
These approaches consider both exploration of new areas and maintaining the consistency of the particle filter approximation.

\paragraph{Contribution} 
We present a new approximation for mutual information that is useful in mobile robotics exploration problems.
The approximation can be easily integrated with forward simulation planning methods, and avoids computing full SLAM filter updates during the planning phase.
In contrast to e.g.~\cite{Martinez-Cantin2009,Indelman2014a,Indelman2015}, we do not assume a Gaussian belief state.
We propose and empirically evaluate in simulated and real-world domains a exploration method combining strengths of decision-theoretic POMDP based exploration and classical frontier based exploration.
In all cases, we concentrate on non-myopic planning instead of the greedy one-step maximization of utility.
\section{Exploration as a POMDP}
\label{sec:pomdp}
Consider a robot exploring a partially observable environment.
Let $s \in \St$ denote the hidden state of the system, comprising the state of the robot and the state of the environment.
At each decision epoch in the set $\Ti=\{0, 1, \ldots, H-1\}$, $H\in\N \cup \{\infty\}$, where $H$ is the horizon of the problem, the robot selects a control action $u\in \U$.
Consequently, the state at the next decision epoch is determined by a transition according to a Markovian state transition model $\T(s',s,u)$, giving the conditional probability of transitioning to $s'\in\St$ from $s\in\St$ when action $u\in\U$ is executed.
After the state transition, the robot obtains information regarding the state in the form of an observation.
The observation is modelled by a probabilistic model $\Ob(z',s',u)$ giving the conditional probability of perceiving observation $z'\in \Z$ in state $s'\in\St$ after action $u\in\U$ was executed.
As the robot's knowledge of the true system state is incomplete, it is represented by a belief state $b\in \B$, a probability density function (pdf) over the state space $\St$. 
The set $\B$, containing all pdfs over $\St$, is called the belief space.

As the robot executes control actions and perceives observations, its belief state is tracked by Bayesian filtering.
Given a belief state $b\in\B$ and an action $u\in\U$, the predicted belief state $b^u$ is computed by 
\begin{equation}
\label{eq:prediction}
b^u(s') = \int\limits_{s\in \St} \T(s',s,u) b(s) \mathrm{d}s.
\end{equation}
Given an observation $z'$, the posterior belief $b'$ is obtained in the update step by applying Bayes' rule:
\begin{equation}
\label{eq:update}
b'(s') := \tau(b,u,z')(s') = \frac{\Ob(z', s', u)b^u(s')}{\eta(z'\mid b, u)},
\end{equation}
where the function $\tau:\B\times \U \times \Z \to \B$ is referred to as the belief update function.
The normalization term in the denominator is the prior probability of observing $z'$: $\eta(z'\mid b, u) = \int\limits_{s'\in \St} \Ob(z', s', u) b^u(s')\mathrm{d}s'$.

For each action, the robot receives an instantaneous reward given by the reward function $\rho:\B\times\U\to\R$.
Throughout the remainder of the paper, we will assume that the reward function is bounded, i.e., there exist $\rho_{\text{min}}, \rho_{\text{max}}\in\R$ such that $\forall b,u: \rho_{\text{min}} \leq \rho(b,u) \leq \rho_{\text{max}}$.
The reward at decision epoch $t\in\Ti$ is discounted by a factor $\gamma^{t}$, typically $\gamma\in[0,1)$ \citep{Puterman1994}.
The robot is said to act optimally when the expected sum of discounted instantaneous rewards over all decision epochs in $\Ti$ is maximized.
The elements thus defined constitute a discounted reward partially observable Markov decision process (POMDP).

\subsection{Reward functions for exploration}
\label{subsec:rewards}
Suitable reward functions for exploration encourage reducing uncertainty in the belief states.
Such reward functions should yield a high reward for beliefs that have the majority of their probability mass concentrated on a small subset of states.
Vice versa, beliefs that correspond to uncertain state information should yield a low reward.
Uncertainty functions and information gain provide useful characterizations that can be applied to define such reward functions.
\begin{definition}[Uncertainty function and information gain~\citep{DeGroot2004}]
\label{def:unc}
An uncertainty function is a concave function $f:\B\to\R^+$.
The information gain $I_f$ related to an uncertainty function $f$ is the expected reduction in the uncertainty function:
\begin{equation}
I_f(b,u) = f(b^u) - \E[f(\tau(b,u,z'))],
\end{equation}
taking the expectation under $\eta(z'\mid b,u)$.
\end{definition}
For example, $I_f(b,u)$ or $-f(b^u)$ may then be applied as a reward function.
A common choice for the uncertainty function $f$ is the Shannon entropy, for which the corresponding information gain is the mutual information~\citep{Cover2006}.
Mutual information has been applied as a reward function in several works on robotic information gathering, see e.g.~\citep{Stachniss2005,Hitz2014,Charrow2014,Atanasov2014,Lauri2015}.

\subsection{Optimal policies}
\label{subsec:opt_policy}
An optimal solution to a finite-horizon POMDP is a sequence of policies $(\pi_1^*, \pi_2^*, \ldots, \pi_{H}^*)$, where $\pi_t^*:\B\to\U$ maps belief states to control actions, such that acting according to $\pi_t^*$ at every decision epoch $t$ maximizes the expected sum of discounted rewards over the remaining decision epochs\footnote{Note that by this formulation, at decision epoch 0, $\pi_H^*$ is followed, at decision epoch $1$, $\pi_{H-1}^*$, and so on.}.
An optimal solution is characterized by the value iteration procedure for $t=1, 2, \ldots, H$ indicating the number of decision epochs remaining, starting with $Q_{1}(b,u) \equiv \rho(b,u)$:
\begin{align}
\label{eq:q_vi}
Q_t(b,u) &= \rho(b,u) + \gamma \int\limits_{z'\in\Z} \eta(z'\mid b, u) V_{t-1}^*(\tau(b,u,z')) \mathrm{d}z',\\
\label{eq:value_vi}
V_t^*(b) &= \sup\limits_{u\in\U} Q_t(b,u),\\
\label{eq:policy_vi}
\pi_{t}^*(b) &= \argsup\limits_{u\in\U} Q_t(b,u).
\end{align}
The value function $V_t^*(b)$ gives the expected sum of discounted rewards attainable by acting optimally for $t$ decision epochs starting at belief state $b$, and $Q_t(b,u)$ gives the expected sum of discounted rewards when action $u$ is first executed and then for the remaining $(t-1)$ decision epochs an optimal policy is followed.
For an infinite horizon problem, value iteration converges to a unique fixed point $V^*$, and the corresponding optimal policy is stationary~\citep{Bertsekas1995}.

Traditionally, reward functions in POMDPs are state and action dependent, and thus linear in the belief state~\citep{Kaelbling1998}.
For $\rho(b,u)$ linear in the belief state, the optimal finite-horizon value function can be represented by a piece-wise linear and convex (PWLC) function~\citep{Smallwood1973}.
Several of the most efficient currently known approximate algorithms for POMDPs leverage the PWLC representation to find solutions~\citep{Hauskrecht2000, Spaan2005, Pineau2006, Kurniawati2008}.

However, a reward function defined as an uncertainty function or information gain (Definition~\ref{def:unc}) is nonlinear in the belief.
\citet{Araya2010} approximated a reward function convex in the belief state by a piece-wise linear function and were able to apply standard POMDP algorithms to approximate the optimal policy.
Online POMDP algorithms \citep{Ross2008} that repeatedly find an optimal action only for the current belief state during task execution and do not rely on an explicit closed form representation of the value function are also applicable to POMDPs with nonlinear rewards.
Several online algorithms are based on a finite depth tree search over reachable belief states.
The number of reachable belief states after $k$ decisions is $(|\U||\Z|)^k$, which quickly makes the search intractable for problems with large action and observation spaces.
If either $\U$ or $\Z$ are uncountable, the basic tree search method is not applicable.

In mixed-observability problems, a part of the state is fully observable and evolves deterministically.
In such mixed-observability domains with a mutual information reward function, \citet{Atanasov2014} showed that if the observation model is linear-Gaussian w.r.t.\ the target state, the optimal policy for maximizing mutual information is open-loop, and if possible actions are constrained by the fully observable part of the state, \citet{Lauri2015} showed that under certain conditions the problem can be relaxed into a multi-armed bandit problem.

\subsection{Open-loop solution to a POMDP}
\label{subsec:openloop}
Instead of resorting to the aforementioned approaches for POMDPs with nonlinear rewards, we consider an open loop approximation that ignores the effect of information provided by future observations, instead choosing actions only on the basis of currently available information.
By this approach, we are able to handle continuous action and observation spaces, although at a loss in optimality of the solution.

Denote the current belief state by $b_0$.
Given an action sequence $u_{0:H-1}$, the predicted pdf of the measurements $z_{1:H}$ given the information in the current belief state is
\begin{equation}
\begin{split}
p(z_{1:H}\mid b_0, u_{0:{H-1}}) = \int\limits_{\St^{|H+1|}} b_0(s_0)\cdot &\\
\prod\limits_{k=0}^{H-1}\Ob(z_{k+1},s_{k+1},u_k)\T(s_{k+1},s_k,u_k)\mathrm{d}s_{0:H}.
\end{split}
\end{equation}
The open loop objective function over the corresponding decision epochs is
\begin{equation}
\label{eq:objectivefcn}
\begin{split}
J_H&(b_0,u_{0:H-1}) = \\
&\rho(b_0,u_0) + \E\left[\sum\limits_{k=1}^{H-1} \gamma^{k} \rho(\tau(b_{k-1},u_{k-1},z_{k}), u_{k})  \right],
\end{split}
\end{equation}
where the expectation is taken w.r.t.\ $p(z_{1:H}\mid b_0, u_{0:{H-1}})$, and belief states follow the recursion $b_k=\tau(b_{k-1},u_{k-1},z_{k})$.
The crucial difference compared to the value iteration and policies in Eqns.~\eqref{eq:q_vi}-\eqref{eq:policy_vi} is that the sequence of actions is \emph{fixed} beforehand.
Thus, an optimal open loop solution $u_{0:H-1}^*$ is a sequence of $H$ actions that maximizes the expected sum of discounted rewards when information provided by observations during the subsequent decision epochs is not exploited to select any of the subsequent actions:
\begin{equation}
\label{eq:optimalopenloop}
u_{0:H-1}^* = \argsup\limits_{u_{0:H-1}\in \U^{|H|}} J_H(b_0,u_{0:H-1}).
\end{equation}
We denote the optimal open loop value as $J_H^*(b_0) \equiv J_H(b_0, u_{0:H-1}^*)$.

A receding horizon control strategy is adopted as follows.
First, Eq.~\eqref{eq:optimalopenloop} is solved, $u_0^*$ is executed and an observation $z_1$ is perceived.
The belief state is revised to $\tau(b_0,u_0^*,z_1)$, and the process is repeated.
In a sense, feedback is reintroduced by computing the optimal open loop action sequence again at every decision epoch, applying updated information.
This control strategy is also referred to as open-loop feedback control (OLFC), and we denote the expected value of following this OLFC policy for $H$ decision epochs as $\hat{J}_H(b_0)$.

It is interesting to consider how much worse is the value $\hat{J}_H(b_0)$ of the OLFC policy compared to the value $V_{H}^*(b_0)$ of an optimal closed loop policy.
If the reward function is bounded, then the loss of OLFC compared to an optimal policy is bounded.
To see this, consider that the worst and best possible policies attain a reward of $\rho_{\text{min}}$ and $\rho_{\text{max}}$, respectively, on every decision epoch.
Taking into account the discounting of rewards, we have that the difference in value of \emph{any} two policies over $H$ decision epochs is at most $\sum\limits_{k=0}^{H-1} \gamma^k (\rho_{\text{max}}-\rho_{\text{min}})$.
It should be noted that in practice an optimal policy rarely obtains the maximum reward on every decision epoch, nor does an open loop policy obtain the minimum reward on every decision epoch, such that the suboptimality is often less in practice than as indicated by the bound.

The OLFC policy always performs at least as well as executing an optimal pure open-loop action sequence $u_{0:H-1}^*$, as shown by the following theorem.
\begin{theorem}[Performance of open-loop feedback control~\citep{Bertsekas2005}]
For any $b\in\B$, 
\begin{equation}
\hat{J}_H(b) \geq J_H^*(b).
\end{equation}
\end{theorem}

The suboptimality of following the OLFC policy may be further analyzed by noting that in certain special cases an optimal policy is effectively an open loop policy, see~\cite{Yu2005} for a discussion of such cases.
Furthermore, an approximation of an upper bound for the optimal closed loop value may be found e.g.\ applying hindsight optimization, see~\cite{Chong2000}.\section{Forward simulation for open loop control}
\label{sec:openloop}
Solving the open loop control problem of Eq.~\eqref{eq:optimalopenloop} corresponds to finding an optimal policy to a POMDP with no observations, which itself is an NP-complete problem~\citep{Papadimitriou1987}. 
Since the optimal solution is a sequence of actions $u_{0:H-1}^*$, it is possible to apply search algorithms to find the solution. 
The search space is $\U^{|H|}$, the space of action sequences with $H$ actions.
As the objective function $J_H$ may not be differentiable with respect to the action sequence, gradient methods are not generally applicable. 
We review an algorithm for both a finite (Subsection~\ref{subsec:mcts}) and an uncountable (Subsection~\ref{subsec:smc}) action space, leading to finite or uncountable search spaces, respectively. 

\subsection{Partially observable Monte Carlo planning}
\label{subsec:mcts}
The partially observable Monte Carlo planning algorithm (POMCP) of \cite{Silver2010} is a variant of Monte Carlo Tree Search (MCTS) \citep{Browne2012} to solve POMDPs with finite action and observation spaces. 
POMCP runs a sequence of forward simulations called episodes on the problem, and collects the rewards obtained in the simulation and uses them to evaluate alternative strategies by constructing a search tree over the possible policies. 
An exploration parameter $e$ determines the balance between exploitation and exploration behaviour during the search. 

\begin{figure}
\centering
\includegraphics[width=\columnwidth]{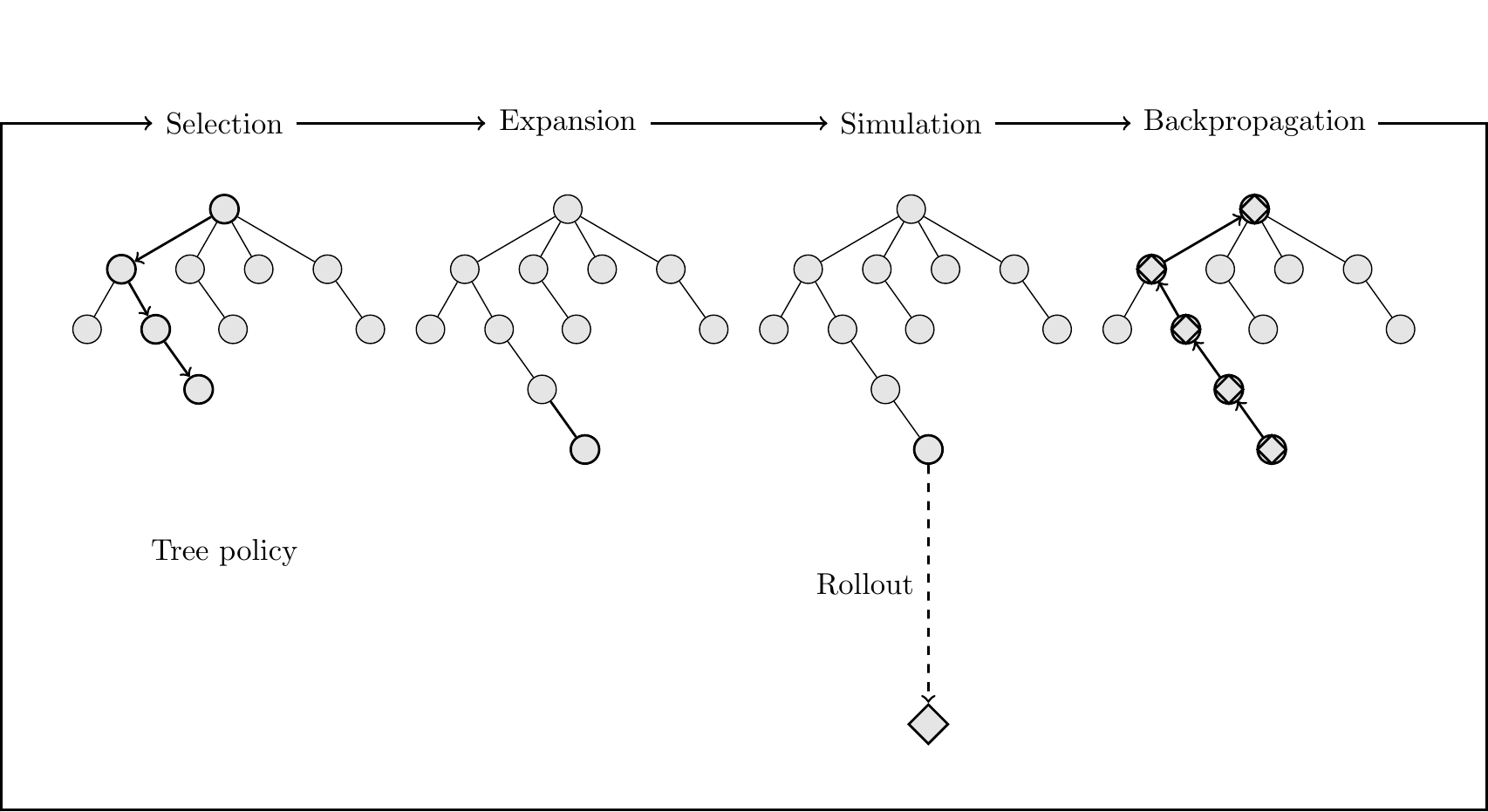}
\caption{Overview of POMCP simulation episodes. The tree policy stage consists of a selection and expansion phase, and the rollout stage consists of a simulation and backpropagation phase. (Figure adapted from~\cite{Browne2012})}
\label{fig:mcts}
\end{figure}

We have adapted POMCP for solving the open loop control problem of Eq.~\eqref{eq:optimalopenloop}.
A search tree over action sequences is iteratively constructed, starting from a tree containing just the root node corresponding to the current belief state $b$.
The overall objective is to estimate the values of possible open loop action sequences, in order to be able to select an action to execute in the current belief state that leads to the most favourable outcome over the next $H$ decisions.

Each node in the search tree corresponds to a particular action sequence, and we will refer to a node by an action sequence $u_{0:k}$.
The root node corresponds to the current belief state $b$ when no actions have yet been taken, and is referred to by an empty set.
For each node $u_{0:k}$, a value estimate and a count are maintained.
The value estimate $V(u_{0:k})$ is the mean sum of discounted rewards recorded over all simulations that continue from $u_{0:k}$ until the final decision epoch $H$.
The count $N(u_{0:k})$ is the number of times a node has been visited during the search.

Figure~\ref{fig:mcts} gives an overview of a simulation episode, and shows how the search tree is iteratively constructed.
Every episode consists of a tree policy stage and a rollout stage.
In the next two paragraphs, we give an informal overview of both phases.
After that, we describe the algorithm in more detail, filling in the missing details.

During the tree policy stage, the current belief state $b$ is propagated applying the belief update equation $\tau$ with each action $u$ determined by the search tree nodes visited and each observation $z'$ determined by sampling from the prior distribution $\eta(z'\mid b,u)$.
For example, at the root node corresponding to $b$, the value estimates and counts of the child nodes are applied to choose one child node corresponding to an action $u$.
An observation is sampled from $\eta(z' \mid b, u)$, the belief state is updated to $b = \tau(b,u,z')$, and the tree policy continues to a new node $u$.
Similarly, at any node $u_{0:k}$ a child node $u_{0:k+1} = \{u_{0:k}, u\}$ is chosen according to the values and counts of the child nodes, the belief state is updated, and the tree policy stage continues at node $u_{0:k+1}$.
An example of a trajectory of nodes traversed during the tree policy stage is indicated by the arrows in the leftmost part of Figure~\ref{fig:mcts}.

Once a leaf node that does not have any children is encountered, the tree policy cannot be continued.
The tree is expanded by adding new child nodes as appropriate, an example of expanding the tree with a single node is shown in the second from left part of  Figure~\ref{fig:mcts}, under the ``Expansion'' indicator.
The search then moves to the rollout stage.
In the rollout stage, a simulation until decision epoch $H$ is performed, updating belief and recording rewards along the way, see the third from left part of Figure~\ref{fig:mcts}.
In the final backpropagation phase of the rollout stage, shown on the rightmost part of Figure~\ref{fig:mcts}, the total reward recorded is applied to update the values of each node visited during the simulation episode.
Once a simulation episode is completed, a new one is again started at the root of the search tree, now expanded with new nodes.

\begin{algorithm}
\caption{POMCP \citep{Silver2010} for open-loop control \label{alg:mcts} }
\begin{algorithmic}[1]
\Require The POMDP model, exploration parameter $e$, optimization horizon $H$, current belief state $b_0$

\Function{POMCP}{$b_0$}
\While{stopping conditions not fulfilled} \label{lin:while}
\State \Call{Simulate}{$b_0, \emptyset, 0$} \label{lin:initsim}
\State $N(\emptyset) \gets N(\emptyset) + 1$
\EndWhile
\State \Return $\argmax_{u} V(u)$ \label{lin:return}
\EndFunction
\Statex

\Function{Simulate}{$b, u_{0:d}, d$}
\If{$d = H-1$} \Return 0 
\ElsIf{\Call{IsLeaf}{$u_{0:d}$}} \label{lin:leaf}
\State Add $\{u_{0:d},u\}$ to tree $\forall u\in \U$ \label{lin:expansion}
\State\Return \Call{rollout}{$b,d$} \label{lin:initrollout}
\EndIf
\State $u \gets \argmax\limits_{u'} V(\{u_{1:d},u'\}) + e \sqrt{{\log N(u_{0:d})}/{N(\{u_{0:d},u'\})}}$ \label{lin:uct}
\State Sample $z' \sim \eta(z'\mid b,u)$
\State $r \gets \rho(b,u) + \gamma \cdot$ \Call{Simulate}{$\tau(b,u,z'),\{u_{0:d},u\}, d + 1$} \label{lin:simulaterecursion}
\State $N(\{u_{0:d},u\}) \gets N(\{u_{0:d},u\}) + 1$ \label{lin:back1}
\State $V(\{u_{0:d},u\}) \gets V(\{u_{0:d},u\}) + \frac{r-V(\{u_{0:d},u\})}{N(\{u_{0:d},u\})}$ \label{lin:back2}
\State\Return $r$
\EndFunction
\Statex

\Function{Rollout}{$b,d$}
\If{$d = H-1$} \Return 0 
\Else
\State Sample $u \sim \pi_{\text{rollout}}(b)$ \label{lin:rolloutpolicy}
\State Sample $z' \sim \eta(z'\mid b,u)$
\State\Return $\rho(b,u) + \gamma \cdot$ \Call{Rollout}{$\tau(b,u,z'),d+1$} \label{lin:rolloutrecursion}
\EndIf
\EndFunction

\end{algorithmic}
\end{algorithm}

Algorithm~\ref{alg:mcts} presents the version of POMCP that we have adapted to solve the open loop control problem of Eq.~\eqref{eq:optimalopenloop}.
The algorithm takes as parameters the POMDP model, an exploration parameter $e$, an optimization horizon $H$, and the current belief state $b$.
The while loop (Line \ref{lin:while}) is executed until a specified stopping criterion, e.g.\ the maximum number of simulation episodes $N_s$, has been reached. 
Each iteration of the while loop corresponds to one simulation episode.

If the condition on Line~\ref{lin:leaf} is false, the current search tree is applied to guide the action selection in the tree policy stage. 
The action to execute is determined by a selection rule maximizing an upper confidence bound~\citep{Auer2002,Kocsis2006} as shown on Line~\ref{lin:uct}\footnote{If $N(\{u_{0:d},u'\})=0$ for any $u'\in \U$, the action is instead sampled uniformly at random from $\lbrace u'\in U \mid N(\{u_{0:d},u'\})=0\rbrace$.}. 
The first term in the bound is the current estimate of the expected value of the action, and the second term is an exploration bonus dependent on the exploration parameter $e$. 
The parameter $e$ is typically selected such that it is on a similar scale as typical rewards in the problem. 
The exploration bonus encourages trying actions that have rarely been tried, prompting exploratory behavior.
After sampling an observation $z'$, the tree policy stage is continued by a recursive call to the function \textsc{Simulate} (Line \ref{lin:simulaterecursion}).
Finally, the counts and values of all action sequences visited during the episode are updated with the reward information gained during the episode in the backpropagation phase (Lines \ref{lin:back1}-\ref{lin:back2}).

If an action sequence is encountered that is a leaf node of the current search tree, a rollout stage is entered (Line~\ref{lin:leaf}).
The search tree is expanded to include possible successor actions (Line~\ref{lin:expansion}).
By calling the function \textsc{Rollout}, a sample of the reward until end of the optimization horizon is obtained.
This is achieved by simulating a given rollout policy (Line~\ref{lin:rolloutpolicy}) until the end of the optimization horizon.
A typical choice of rollout policy is to sample an action uniformly at random from $\U$~\citep{Silver2010}.

Once the while loop of the main function terminates, an action recommendation $\hat{u} \gets \argmax_{u} V(u)$ is returned (Line \ref{lin:return}).
The value estimates of POMCP converge asymptotically to the optimal open loop value~\citep{Silver2010}. 

\subsection{Sequential Monte Carlo planning}
\label{subsec:smc}
\citet{Kantas2009} present a sequential Monte Carlo (SMC) method to maximizing functions of the form
\begin{equation}
\label{eq:smcmax}
J_H(\phi) =  \int\limits_{\mathcal{Y}} G(y,\phi)p(y\mid \phi) \mathrm{d}y,
\end{equation}
which is seen to be equivalent to Eq.~\eqref{eq:objectivefcn} by setting $y=z_{1:H}$, $\phi=\{b_0, u_{0:H-1}\}$, and choosing $G$ as a function that encompasses the appropriate summations\footnote{For this method, $G$ must be finite and strictly positive, which can be achieved by adding a finite positive constant without affecting the argument maximizing the objective function.}.

The key idea of the SMC optimization method of~\cite{Kantas2009} is to construct a sequence of distributions $\lambda_\nu \propto p(\phi)J_H(\phi)^\nu$ where $p(\phi)$ is an arbitrary prior that is nonzero at the maximizers of Eq.~\eqref{eq:smcmax}.
This approach is related to the Bayesian interpretation of simulated annealing and maximum likelihood estimation, see \cite{Johansen2008}. 
As $\nu \to \infty$, $\lambda_\nu(\phi)$ becomes concentrated on the set of maximizers of $J_H$. 
In practice, a set of particles is evaluated via the objective function and, based on the evaluation result, particles are either discarded or carried forward to the next evaluation round. 
Eventually the particle set converges to represent an action sequence maximizing Eq.~\eqref{eq:smcmax}. 

The SMC method is presented in Algorithm~\ref{alg:smc}.
The algorithm iterates over $l=1,\ldots,l_{max}$, where the number of iterations $l_{max}$ is chosen according to the accuracy and run-time requirements. 
The algorithm maintains a set of $M$ particles. 
At iteration $l$, each particle indexed by $i$ with an associated weight $w_l^{(i)}$ represents a sequence of actions $u_{0:H-1,l}^{(i)}$ and a set of $\nu_l$  observation sequences conditional on the action sequence, denoted $z_{1:H,j}^{(i)}$, $j=1,\ldots,\nu_l$. 
The integers $\lbrace\nu_l\rbrace_{l\geq 1}$ must form a strictly increasing sequence. 
The algorithm has been shown to converge to the optimal solution for logarithmic sequences, although in practice faster increasing linear or geometric sequences are used \citep{Johansen2008, Kantas2009}. 
The sequence $\lbrace\nu_l\rbrace_{l\geq 1}$ is analogous to the temperature cooling schedule in simulated annealing.

\begin{algorithm}
\caption{Sequential Monte Carlo optimization, adapted from \citep{Kantas2009}. \label{alg:smc} }
\begin{algorithmic}[1]
\Require The POMDP model, number of iterations $l_{max}$, increasing integer sequence $\lbrace \nu_l \rbrace_{l=1}^{l_{max}}$, sampling kernels $q_l$, optimization horizon $H$, resampling threshold $M_t$.
\Function{SMC}{$b_0$}
\For{$l = 1, \ldots, l_{max}$}
\For{$i=1,\ldots,M$}
\State $u_{0:H-1,l}^{(i)} \sim q_l(\cdot \mid u_{0:H-1,l-1})$ \label{lin:actionsample}
\For{$j=1, \ldots, \nu_l$} \label{lin:replicastart}
\State $z_{1:H,j}^{(i)}\sim p(z_{1:H}\mid b_0, u_{0:H-1,l}^{(i)})$
\EndFor \label{lin:replicastop}
\State \begin{varwidth}[t]{\linewidth}$w_{l}^{(i)} \gets$\par
        \hskip\algorithmicindent $w_{l-1}^{(i)}\prod\limits_{j=1}^{\nu_l} \left(\rho(b_0,u_{0,l}^{(i)}) \right.$,\par
        \hskip\algorithmicindent $\left. + \sum\limits_{k=1}^{H-1} \gamma^{k} \rho(\tau(b_{k-1}^{(i)},u_{k-1,l}^{(i)},z_{k,j}^{(i)}), u_{k,l}^{(i)})  \right)$ \label{lin:evaluation}
        \end{varwidth}
\EndFor
\For{$i=1,\ldots,M$}
\State $w_{l}^{(i)} \gets {w_l^{(i)}}/{\sum\limits_{j=1}^{M}w_l^{(j)}}$
\EndFor
\If{$1/ \sum\limits_{i=1}^{M} (w_l^{(i)})^2 < M_t$}
\State Resample; $w_{l}^{(i)} \gets 1/M \quad \forall i$ \label{lin:resample}
\EndIf
\EndFor
\EndFunction
\end{algorithmic}
\end{algorithm}

A control sequence is sampled for each particle from a kernel $q_l$ (Line \ref{lin:actionsample}).
For $l=1$, the initial control sequences are sampled from a prior distribution that is nonzero at the maximizers of \eqref{eq:smcmax}.
In general, the choice of $q_l$ depends on the problem specifics. 
The selection of the kernel determines how new control sequence samples depend on the earlier ones, and it has a central role in determining how efficient the algorithm is. 
For further discussion on kernel selection, we direct the reader to \cite{DelMoral2006}.

After control sequences have been sampled, they are applied to generate $\nu_l$ replicas of possible state-observation sequences when executing the control sequence (Lines \ref{lin:replicastart}-\ref{lin:replicastop}). 
The weights of the particles are updated by approximating the objective function applying the state and observation samples contained in each particle (Line \ref{lin:evaluation}).
In this step, the belief states for given $j,l,i$ follow the recursion $b_k^{(i)} = \tau(b_{k-1}, u_{k-1,l}^{(i)},z_{k,j}^{(i)})$.

\begin{figure}[t!p]
\centering
	\begin{subfigure}[t]{0.33\columnwidth}
        \includegraphics[width=\textwidth]{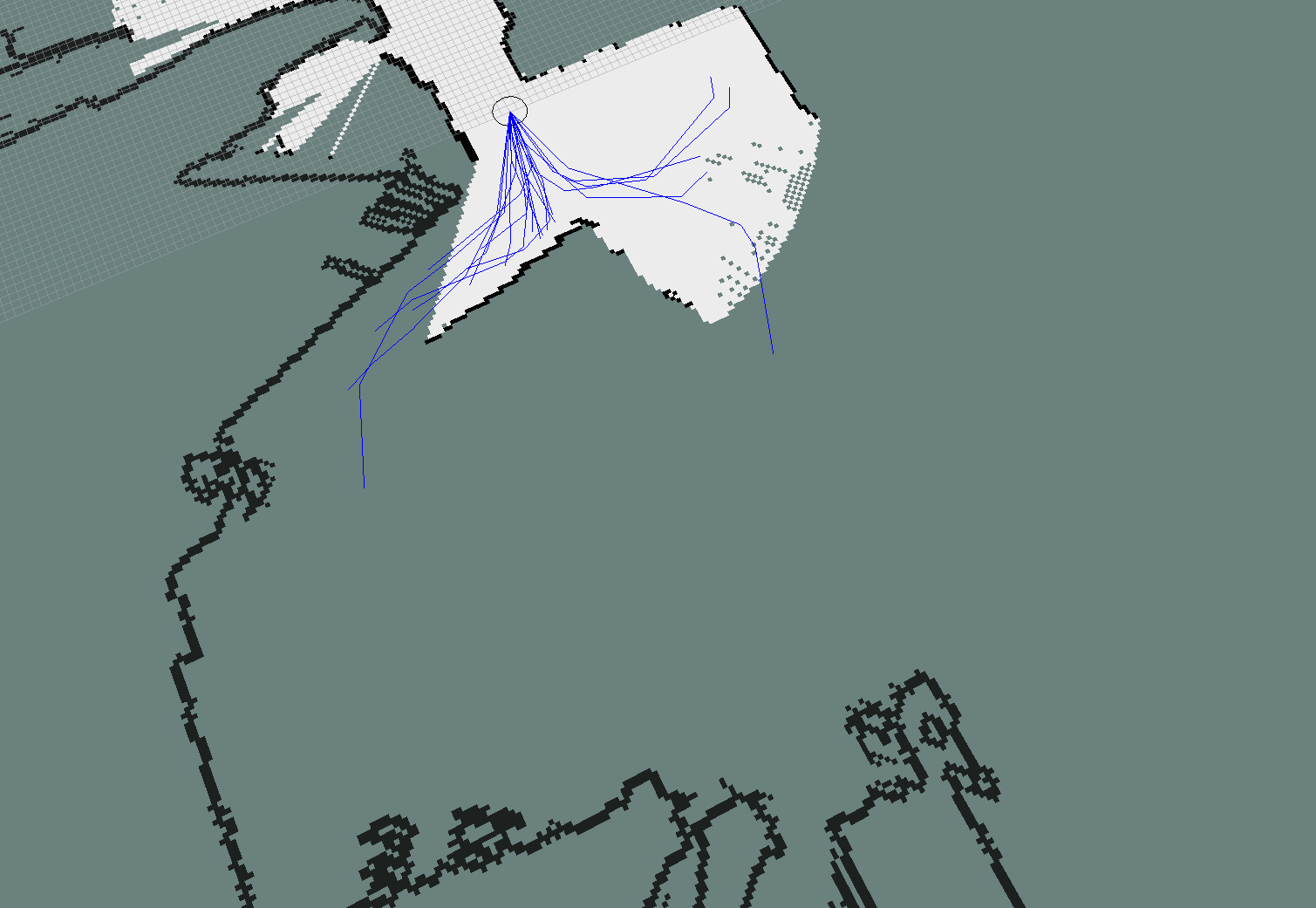}
        \caption{}
        \label{fig:smc_convergence_a}
    \end{subfigure}
    \begin{subfigure}[t]{0.33\columnwidth}
        \includegraphics[width=\textwidth]{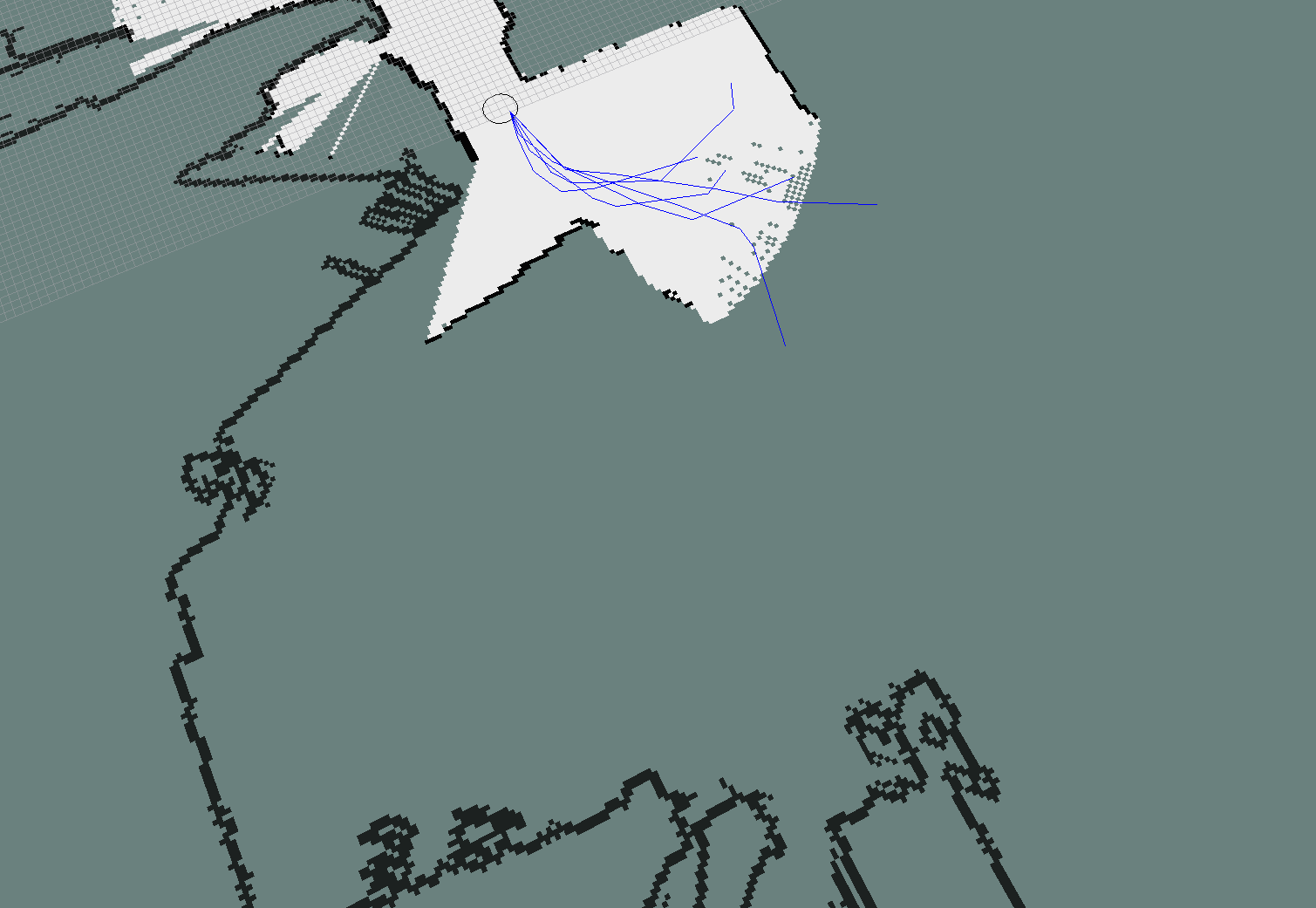}
        \caption{}
        \label{fig:smc_convergence_b}
    \end{subfigure}
    \begin{subfigure}[t]{0.33\columnwidth}
        \includegraphics[width=\textwidth]{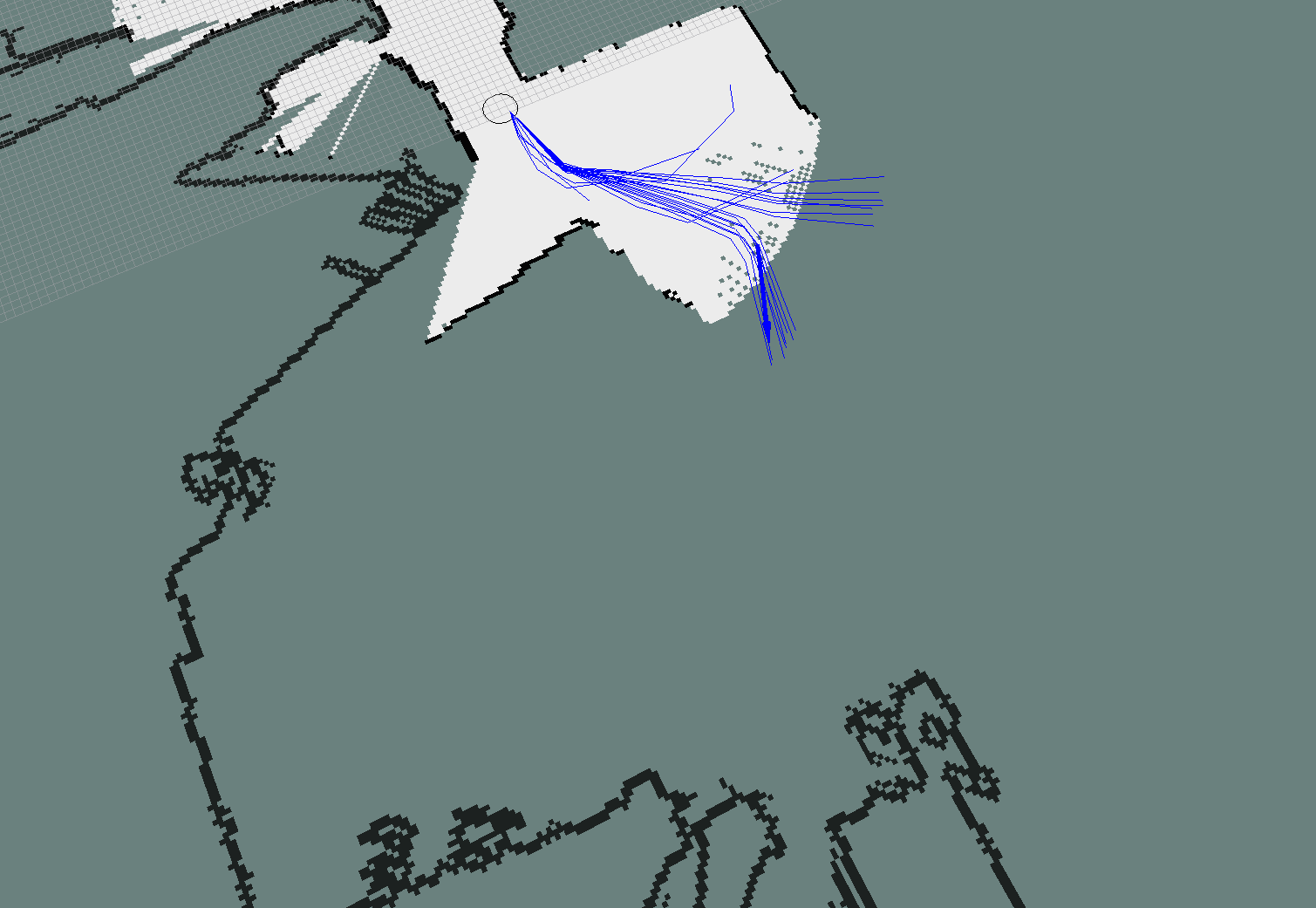}
        \caption{}
        \label{fig:smc_convergence_c}
    \end{subfigure}
\caption{Convergence of the SMC algorithm in a robotic exploration task. Subfigures \ref{fig:smc_convergence_a} through \ref{fig:smc_convergence_c} indicate increasing iterations $l$ of the algorithm. The robot location and outline are indicated by a thin black circle. Local trajectories shown by the blue lines are superimposed on a greyscale image indicating current map information. The map shows occupied cells in black, free cells in white, and unknown cells in grey.  The local trajectories are sequentially evaluated based on their informativeness, and eventually converge towards the most informative local trajectory.}
\label{fig:smc_convergence}
\end{figure}

At the end of each iteration, resampling is performed if the effective sample size \citep{Liu1998} falls below the threshold value $M_t$ (Line \ref{lin:resample}).
After step $l=l_{max}$, the maximizer estimate is extracted as $u_{1:H-1,l_{max}}^{(i_m)}$, where $i_m = \argmax\limits_{i}w_l^{(i)}$. 
Algorithm~\ref{alg:smc} is easily parallelized, as each particle can be processed independently.

Figure~\ref{fig:smc_convergence} shows an example of applying Algorithm~\ref{alg:smc} to a robotic exploration problem.
The figure shows for three iterations of the algorithms how possible trajectories for the robot corresponding to the particles are evaluated, weighted and resampled, converging towards the most informative local trajectory.\section{Mutual information in mobile robotics}
\label{sec:miapprox}
In this section, we present a new approximation for mutual information (MI) that is especially useful in mobile robotics domains.
As the approximation is sample-based, it is straightforward to apply it together with forward simulation based planning.
The approximation is derived in Subsection~\ref{subsec:mi}, and in Subsection~\ref{subsec:sampling} we present a method for efficiently drawing observation samples in the case occupancy grids are used as a map representation.

\subsection{Approximation of mutual information}
\label{subsec:mi}
The state $s\in\St$ in a robotic exploration problem is decomposed into the robot's internal state $x\in \X$ and the environment state or map $m\in \M$, i.e.\ $\St=\X\times \M$. 
The internal state represents the pose of the robot, and the environment state represents all relevant features of the world. 
We make the following assumption asserting that the robot is unable to affect the environment state through its actions. 
\begin{assumption}
\label{asm:independence}
The robot's internal state and environment state evolve independently, i.e.\
\begin{equation}
\label{eq:indassumption}
\T(s', s, u) = \T_x(x', x, u)\T_m(m',m),
\end{equation}
where $\T_x:\X\times\X\times\U\to\R^+$ and $\T_m:\M\times\M\to\R^+$ are the robot and environment state transition models, respectively.
\end{assumption}
We remark that this assumption does not imply that the internal and environment state are independent.

Suppose information gain (Definition~\ref{def:unc}) with Shannon entropy as the uncertainty function is applied as the reward in an exploration POMDP.
This leads to $I_f(b,u)$ equal to the mutual information, which can be approximated as follows.
\begin{theorem}
\label{thm:mi}
Denote by $b\in\B$ the current belief state, by $u\in\U$ the current action, and let $X$, $M$, and $Z$ denote the random variables depicting the robot state, environment state, and observation at the next decision epoch.
Consider the mutual information of the state $(X,M)$ and observation $Z$, defined
\begin{equation}
\begin{split}
I&(X,M;Z\mid b, u) =\\
 \int\limits_{\M}\int\limits_{\X} p(x',m'\mid b,u)\log&\frac{p(x',m'\mid b,u)}{p(x'\mid b,u)p(m'\mid b,u)}\mathrm{d}x'\mathrm{d}m'.
\end{split}
\end{equation}
If Assumption~\ref{asm:independence} holds and if $\{x'^{(i)}, z'^{(i)}\}_{i=1}^N \sim p(x',z'\mid b, u)$, where $N$ is the number of samples, then the approximation
\begin{equation}
\label{eq:mi_approximation}
\begin{split}
I_N &= \frac{1}{N}\sum\limits_{i=1}^N \Bigg[\log \frac{p(x'^{(i)}\mid b, u, z')}{p(x'^{(i)}\mid b,u)} \\
 + \int\limits_{\M}p(m'\mid b,u,z'^{(i)}, x'^{(i)})&\log \frac{p(m'\mid b, u, z'^{(i)}, x'^{(i)})}{p(m'\mid b, u)} \mathrm{d}m' \Bigg]
\end{split}
\end{equation}
converges almost surely to $I(X,M;Z\mid b, u)$ as $N\to\infty$.
\end{theorem}
\begin{proof}
By the chain rule for MI \citep{Cover2006},
\begin{equation}
\label{eq:michain}
I(X,M; Z\mid b, u) = I(X; Z \mid b, u) + I(M; Z \mid X, b, u).
\end{equation}
The first term of Eq.~\eqref{eq:michain} is the MI of $X$ and $Z$:
\begin{equation}
\label{eq:mi1}
I(X;Z\mid b, u) = \int\limits_{\Z}\int\limits_{\X} p(x',z'\mid b,u) \log \frac{p(x'\mid b,u,z')}{p(x'\mid b,u)} \mathrm{d}x'\mathrm{d}z'.
\end{equation}
The second term of Eq.~\eqref{eq:michain} is the conditional MI of $M$ and $Z$, given $X$.
By the definition of conditional MI \citep{Cover2006},
\begin{equation}
I(M;Z\mid X, b, u) = \E_{Z,X,M}\left[ \log \frac{p(m',z'\mid x',b,u)}{p(m'\mid x',b,u)p(z'\mid x',b,u)}\right],
\end{equation}
where the expectation is taken w.r.t.\ $p(z',x',m'\mid b,u)$.
By Assumption~\ref{asm:independence}, $p(m'\mid x',b,u) = p(m'\mid b,u)$, and by the law of conditional probability we have
\begin{equation}
\label{eq:mi2}
\begin{split}
&I(M;Z\mid X,b,u) = \int\limits_{\Z}\int\limits_{\X}p(x',z'\mid b,u)\\
 &\left[\int\limits_{\M}p(m'\mid x',z',b,u)\log\frac{p(m'\mid x',z',b,u)}{p(m'\mid b,u)}\mathrm{d}m' \right]\mathrm{d}x'\mathrm{d}z'
\end{split}
\end{equation}
We have thus shown that both terms on the right hand side of Eq.~\eqref{eq:michain}, namely, Eqns.~\eqref{eq:mi1}~and~\eqref{eq:mi2}, are expectations under the same joint pdf $p(x',z'\mid b,u)$.
Drawing $N$ samples from this pdf and defining $I_N$ as above, by the lemma of Monte Carlo integration \citep[Ch.~3.2]{Robert1999} $I_N$ converges to $I(X,M;Z\mid b, u)$ almost surely.
\end{proof}

To apply the approximation presented in Theorem~\ref{thm:mi}, we draw samples from $p(x',z'\mid b,u)$ by the following three-step method for $i=1,2,\ldots,N$:
\begin{enumerate}
\item Sample from the current belief state $s^{(i)} = (x^{(i)}, m^{(i)}) \sim b$,
\item Propagate the sample through the factored state transition model of Assumption~\ref{asm:independence}, yielding new samples $x'^{(i)}\sim \T_x(x',x^{(i)},u)$ and $m'^{(i)} \sim \T_m(m',m^{(i)})$.
\item The sample from the previous step is distributed according to the predictive pdf $b^u = p(x',m'\mid b,u)$. Recalling $s' = (x', m')$, we have by marginalizing over $\M$
\begin{equation}
p(x',z'\mid b,u) = \int\limits_{\M}\Ob(z',s',u)b^u(s')\mathrm{d}m',
\end{equation}
and we may thus sample $z'^{(i)} \sim \Ob(z',s'^{(i)},u)$ to obtain the desired samples.
\end{enumerate}

Theorem~\ref{thm:mi} is applicable with three restrictions.
First, we must be able to draw samples from the state transition and observation models either directly or e.g.\ applying importance sampling. 
Secondly, evaluating the first sum term in Eq.~\eqref{eq:mi_approximation} requires us to be able to evaluate the predictive probability and the posterior probability of a robot state $x'^{(i)}$ given an action $u$ and a measurement $z'^{(i)}$. 
The predictive pdf can be evaluated applying the state transition model. 
It is not generally easy to find the posterior $p(x'^{(i)} \mid b,u,z'^{(i)})$, but in some cases it can be approximated by a unimodal distribution.
As argued by~\cite{Grisetti2007}, such a case arises e.g.\ for robots equipped with accurate range sensors such as laser range finders (LRFs). 
LRF data can be leveraged via scan matching to obtain localization estimates that are significantly more precise than estimates based only on robot's state transition models. 
Consequently, the measurement likelihood is strongly peaked and the posterior pdf of the robot state given the observation may be approximated by a unimodal distribution, e.g.\ a Gaussian with a small covariance.
Third and finally, evaluating the integral term in the sum of Eq.~\eqref{eq:mi_approximation} requires us to be able to compute posterior map pdfs given a pose and an observation.
This corresponds to the problem of mapping with known poses~\citep{Moravec1988}. 
We must also be able to evaluate the integral expression over the map state. 
The difficulty of this depends on the map representation, but is simple e.g.\ for occupancy grid maps due to the independence assumption of the grid cells: the integral term reduces to a sum over cells.

The overall effect of the approximation is that we avoid computing the full state posterior pdf and solving the implied SLAM problem when evaluating the expected information gain of a control action.
Note that this does not preclude executing a SLAM algorithm for the real process independently of the task of finding an optimal open loop action sequence. 
The current belief state $b=p(x,m)$ provided by the SLAM algorithm is applied as initial information for finding the optimal open loop solution. 
Sometimes the real process SLAM particle filter may be susceptible to losing consistency~\citep{Blanco2008, Carlone2010}, and applying the proposed approximation may lead to a failure state as this possibility is not considered while planning trajectories. 
In such cases, either the SLAM algorithm must be improved or an alternative approximation for MI applied.

\subsection{Observation sampling in occupancy grid maps}
\label{subsec:sampling}
To apply the forward simulation based planning algorithms presented in Section~\ref{sec:openloop} together with the approximation of MI presented above, efficient methods for drawing observation samples given a sequence of control actions are required.
In this section, we introduce such a sampling method for occupancy grid maps \citep{Moravec1988}, a map representation widely applied in mobile robotics.

An occupancy grid map defines a partition of a space into equally-sized cells $c$.
We assume that the map $\M$ is composed of a finite number of such cells.
Each cell is in one of two hidden states, either free (0) or occupied (1).
As cell occupancy is often not known exactly, an occupancy probability $P(c=1):=p_c \in [0,1]$ models the information regarding the state of the cell.

The occupancy information is revised through sensor observations.
Depending on the robot's sensors and its pose $x$, it is possible that only a subset $\tilde{\M}(x)\subset \M$ of the map is sensed by a robot at $x$.
Figure~\ref{fig:sensing} shows an example of such a scenario.
We note that even several measurements over a trajectory traversed by the robot may only provide information regarding a subset of the whole map area.
As the cells of an occupancy grid map are independent, this suggests that an efficient method for sampling observations may be constructed by considering only the subset of potentially sensed map cells.

\begin{figure}[tp!]
\centering
\includegraphics[width=\columnwidth]{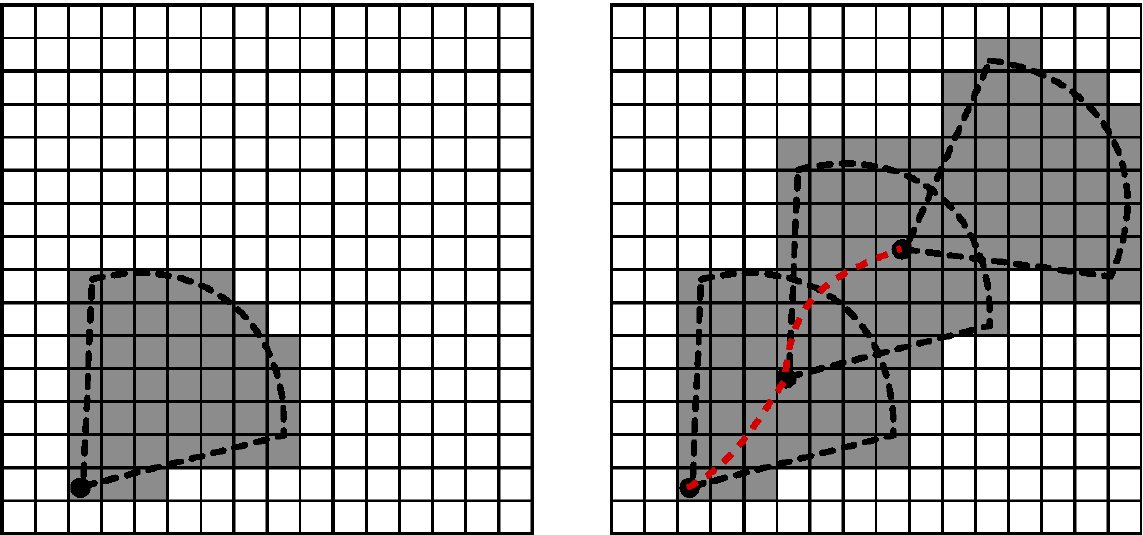}
\caption{Occupancy grids and observed areas. The robot's pose is marked by a black circle, and the dashed sector depicts sensor range. The cells with a grey background denote the subset of the occupancy grid that is potentially sensed either at a single pose (left side) or a trajectory depicted by dashed lines between robot poses (right side).}
\label{fig:sensing}
\end{figure}

We represent observations $z$ as collections of pairs of the form $(c,j)$, where $c$ denotes the cell and $j\in\{0,1\}$ is the \emph{observed occupancy} for the cell: free (0) or occupied (1).
Given an optimization horizon $H$ and the control actions $u_{0:H-1}$, a sequence $z_{1:H}^{(i)}$ of observation samples is constructed.
In the case of a laser range finder, raytracing may be applied to draw observation samples.
Raytracing simulates a laser beam originating from the robot's pose $x$, tracking which cells the beam travels through.
When the ray intersects a cell $c$, the occupancy information $p_c$ is applied to determine whether the cell is free and the beam passes through, or the cell is occupied and the beam hits an obstacle in the cell.
In the former case, $(c,0)$ is inserted to $z_t^{(i)}$, and in the latter case $(c,1)$ is inserted to $z_t^{(i)}$ and the raytracing is terminated.

As the same cell may be observed at multiple decision epochs and the observations should be consistent, i.e. they are generated from the same true underlying map, a persistent map sample $m^p$ keeping track of the occupancy states of cells intersected by the simulated laser beams is maintained.
The persistent map sample is also a collection of pairs $(c,o)$ of a cell $c$ and its \emph{sampled occupancy} $o\in\{0,1\}$ in the map.
Whenever a cell $c$ is encountered during ray tracing, it is first checked whether $m^p$ contains it already.

An observation sampling algorithm implementing these features is presented in Algorithm~\ref{alg:sampling}.
A sample of the current pose $x_0^{(i)}$ is drawn from $p(x_0)$ obtained by marginalizing from the current belief state $b_0=p(x_0,m_0)$ (Line \ref{lin:initsample}).
For each decision epoch $0\leq t\leq H-1$, a pose sample $x_{t+1}^{(i)}$ is drawn (Line~\ref{lin:getactionandsample}) from the robot's dynamic model.
Raytracing with the persistent map sample $m^p$ is applied to sample $z_{t+1}^{(i)}$ (Line~\ref{lin:sampleobs}) from the observation model.
In this case, the sampling is implemented in the function \textsc{Raytrace}.

For each laser beam incidence angle and each cell along the beam, the occupancy state of the cell is first determined.
If the cell $c$ has already been encountered during the sampling of any $z_k^{(i)}$, $k<t$, its occupancy state is retained from the persistent map sample, and otherwise its occupancy state is sampled according to $p_c$ (Lines~\ref{lin:inmap_start}-\ref{lin:inmap_stop}).
Finally, the actual observed occupancy $j$ is sampled according to the sensor observation model (Line~\ref{lin:samplecell}), and the resulting observation is added to the observation sample.
The value $j=1$ indicates that the beam hit an obstacle, and the raytrace for this incidence angle is terminated (Line~\ref{lin:samplebreak}).

\begin{algorithm}
\caption{Observation sampling for a laser range finder (LRF). \label{alg:sampling}}
\begin{algorithmic}[1]
\Require The POMDP model, optimization horizon $H$.

\Function{sample}{$b_0$}
\State $m^p \gets \emptyset$
\State Sample $x_0^{(i)} \sim p(x_0)$ \label{lin:initsample}
\For{$t = 0, \ldots H-1$}
\State Sample $x_{t+1}^{(i)} \sim \T_x(x_{t+1} \mid x_{t}^{(i)}, u_t)$ \label{lin:getactionandsample}
\State $z_{t+1}^{(i)} \gets$ \Call{raytrace}{$x_{t+1}^{(i)}$} \label{lin:sampleobs}
\EndFor
\State \Return $z_{1:t}^{(i)}$
\EndFunction
\Statex

\Function{Raytrace}{$x$}
\ForAll{incidence angles $\alpha$ of the LRF}
\ForAll{cells $c$ along a beam starting at $x$ in direction $\alpha$}
\If{$(c,\cdot) \notin m^p$} \label{lin:inmap_start}
\State Sample $o \sim p_c$
\State $m^p \gets m^p \cup \{(c, o)\}$\label{lin:inmap1}
\Else
\State $o \gets \{o \mid (c,o) \in m^p\}$
\EndIf \label{lin:inmap_stop}
\State Sample $j \sim p(z' \mid o, x)$ \label{lin:samplecell}
\State $z \gets z \cup \{(c,j)\}$
\If{$j = 1$}
\State \textbf{break} \label{lin:samplebreak}
\EndIf
\EndFor
\EndFor
\State \Return $z$
\EndFunction
\end{algorithmic}
\end{algorithm}

As map occupancies are sampled only as required (when a simulated LRF ray enters a cell), computational savings are accrued compared to the naive approach of always sampling occupancy values for every cell before raytracing.
After the observation samples $z_{1:H}^{(i)}$ have been obtained, the corresponding values of MI can be computed applying the approximation of Theorem~\ref{thm:mi}.
For integration with the SMC method (Algorithm~\ref{alg:smc}), the sampling algorithm (Algorithm~\ref{alg:sampling}) is executed independently for the control action sequence of each particle.
For integration with the POMCP method (Algorithm~\ref{alg:mcts}), the action $u_t$ on Line~\ref{lin:getactionandsample} of Algorithm~\ref{alg:sampling} is obtained either from the tree policy or from the rollout policy.\section{Simulated exploration experiments}
\label{sec:simulation}
We studied the POMDP based planning in a set of exploration experiments in three simulated domains.
In particular, the effect of prior information and the optimization horizon $H$ were examined.
The proposed planning approaches were compared to myopic (one-step greedy) planning, which is the special case of the proposed approach with $H=1$ (see e.g.~\cite{Stachniss2005,Sim2005,Kollar2008}), and frontier-based exploration~\citep{Yamauchi1997}.

We begin by considering an illustrative toy example in Subsection~\ref{subsec:toy_example}.
In Subsection~\ref{subsec:fwdsim_performance}, we study the performance of the proposed POMDP based planning approach in three domains: a maze, an office, and an outdoor environment, each having different scale and other properties.
These simulations were implemented in MATLAB.
Based on the results, we propose a further improvement of the POMDP approach, and compare the improved method with frontier-based exploration in Subsection~\ref{subsec:frontier_comparison}.

\subsection{An illustrative example}
\label{subsec:toy_example}
We first consider the effect of prior information and optimization horizon for making a single decision in a small toy domain as shown in Figure~\ref{fig:simple}.
The white and black cells in the map are unoccupied and occupied cells, respectively.
The robot was equipped with a simulated laser range finder with a maximum range of 2 metres for sensing the environment state, and had three possible action choices to select between as indicated by the trajectories overlaid in the figure.
Each action was designed specifically to have a varying effect on the achievable mutual information over a horizon of several decisions.
Action $a_2$, corresponding to moving straight ahead, initially does not provide much information about the environment as the view of the range finder is partially blocked by the corridor leading towards the top of the figure.
Despite this, executing $a_2$ eventually leads the robot to the large open area in the top part of the environment.
In contrast, actions $a_1$ and $a_3$ yield greater immediate information gain, but ultimately lead to dead-ends at the left and right bottom parts of the map.
The boundary of the area was assumed to be known to the robot, i.e., the robot can detect when an action would take it outside the boundary of the environment.

\begin{figure}
\centering
\includegraphics[width=\columnwidth]{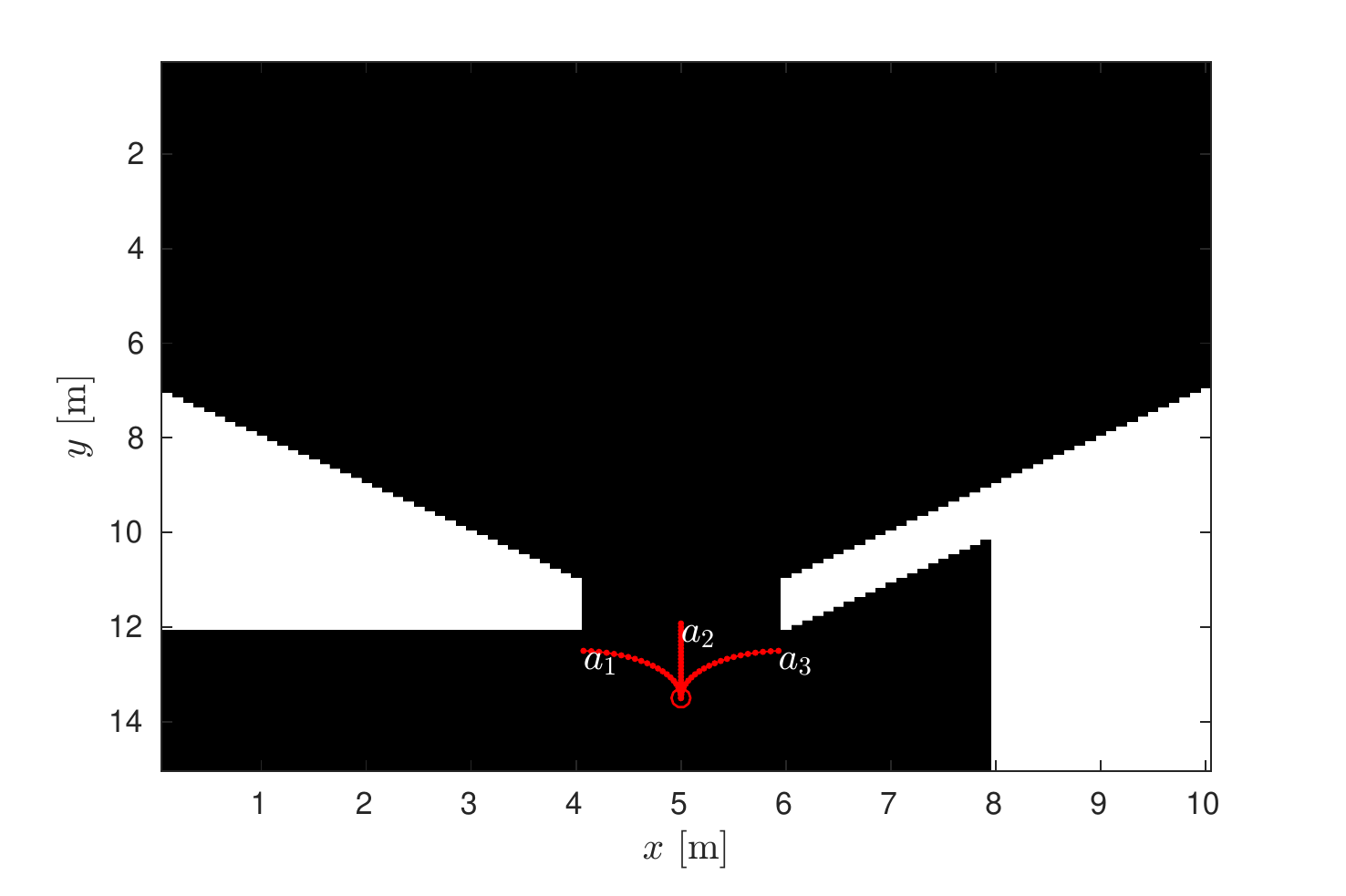}
\caption{A test map for a single decision. The robot's starting location is indicated by the circle marker, and three trajectories are labeled by the corresponding actions $a_1$, $a_2$ and $a_3$.}
\label{fig:simple}
\end{figure}

We note that in a real application also intermediate choices with a curvature between that of $a_1$ and $a_2$, or $a_2$ and $a_3$, would be available for the robot.
However, for the purpose of this example, we did not consider them.
In the subsequent experiments on larger maps, the SMC method considers the full uncountable action space.

As prior information, the robot recorded a single observation with the laser range finder from the initial location, and updated the map occupancy probabilities accordingly.
The robot's initial location was selected so that this observation did not reveal the presence of the occupied cells in the map.
Additionally, two cases were considered: one with no additional prior information, and another one with an informative prior that indicated the occupied areas in the environment (the white cells in Figure~\ref{fig:simple}), but leaving occupancy probabilities elsewhere as they were.

\begin{table}[ht]
\centering
\caption{Values of actions in the small map as function of the optimization horizon $H$ and for a non-informative and informative prior, shown with the true optimal action (bottom row).}
\label{tab:simple}
\begin{tabular}{cccccc}
\toprule
                 Prior                      & Action  & $H=1$ & $H=2$ & $H=3$ & $H=4$   \\ \midrule
\multirow{3}{*}{Uniform}   & $a_1$ & 102.4    & 173.6    & 235.7    & 212.4  \\
                                       & $a_2$ & 101.8    & 167.1    & \textbf{248.8}    & \textbf{325.4}   \\
                                       & $a_3$ & \textbf{106.5}    & \textbf{177.3}    & 235.0    & 223.0  \\ \midrule
\multirow{3}{*}{Informative}     & $a_1$ & 81.8    & 145.9    & 189.7    & 200.5    \\
                                       & $a_2$ & 78.9    & \textbf{157.0}    & \textbf{231.0}    & \textbf{298.5}     \\
                                       & $a_3$ & \textbf{97.2}    & 129.4    & 172.5    & 140.8    \\ \midrule
Opt. action                         &       & $a_3$    & $a_2$    & $a_2$    & $a_2$    \\ \bottomrule
\end{tabular}
\end{table}

We applied the POMCP algorithm (Algorithm~\ref{alg:mcts}) to plan a single decision.
Table~\ref{tab:simple} shows the value estimate of each action for either case of prior information as a function of the optimization horizon $H$.
The action with the greatest value is indicated by a bold font.
Also the optimal action, as found by an exhaustive search over all possible action sequences while assuming perfect sensing, is shown for reference.
As $H$ increases, action $a_2$ is chosen in the case of no additional prior for $H\geq 3$ and in case of the informative prior for $H\geq 2$.
In case of no additional prior information, $a_2$ is eventually preferred as the other actions eventually lead (after two decision epochs in the planning phase) to a dead-end where the robot cannot progress any further and thus cannot collect more information.
In the case of an informative prior, however, $a_2$ is preferred already for $H=2$: as the occupied cells are indicated, the robot is able to avoid the dead-ends blocked by the occupied cells near the bottom corners of the map.
We also note that the difference in the values between the recommended and second-best action tend to be greater for the informative prior, indicating that the algorithm is able to more confidently distinguish between the actions.

\subsection{Performance of POMDP based planning}
\label{subsec:fwdsim_performance}
We next examine the effect of prior information and the optimization horizon in domains larger than the toy example presented above.
Subsection~\ref{subsubsec:sim_setup} outlines the experimental setup, and Subsection~\ref{subsubsec:results} presents the results.

\subsubsection{Experimental setup}
\label{subsubsec:sim_setup}
Three environments as illustrated in Figure~\ref{fig:environments} were examined: a maze (Figure~\ref{fig:maze}), office-like (Figure~\ref{fig:mit}), and an open outdoor environment (Figure~\ref{fig:fr}).
White and black cells are unoccupied and occupied, respectively, and gray cells indicate undefined or unobserved cells.
The sizes of the environments varied from approximately 10-by-10 meters (maze) to 250-by-250 meters (outdoor).
The office and outdoor maps were obtained from an online data repository\footnote{Robotics Data Set Repository Radish~\citep{Radish}. We acknowledge Cyrill Stachniss (and Giorgio Grisetti) for providing the office (outdoor) environment data.}.
The $x$ and $y$ coordinates of the robot's starting location in each environment were as follows.
In the maze environment, $(x,y) = (7.9, 8.6)$; in the office environment, $(38.5,58.0)$; and in the outdoor environment, $(90.0, 106.0)$.

\begin{figure}[t!p]
\centering
	\begin{subfigure}[t]{\columnwidth}
        \includegraphics[width=\textwidth]{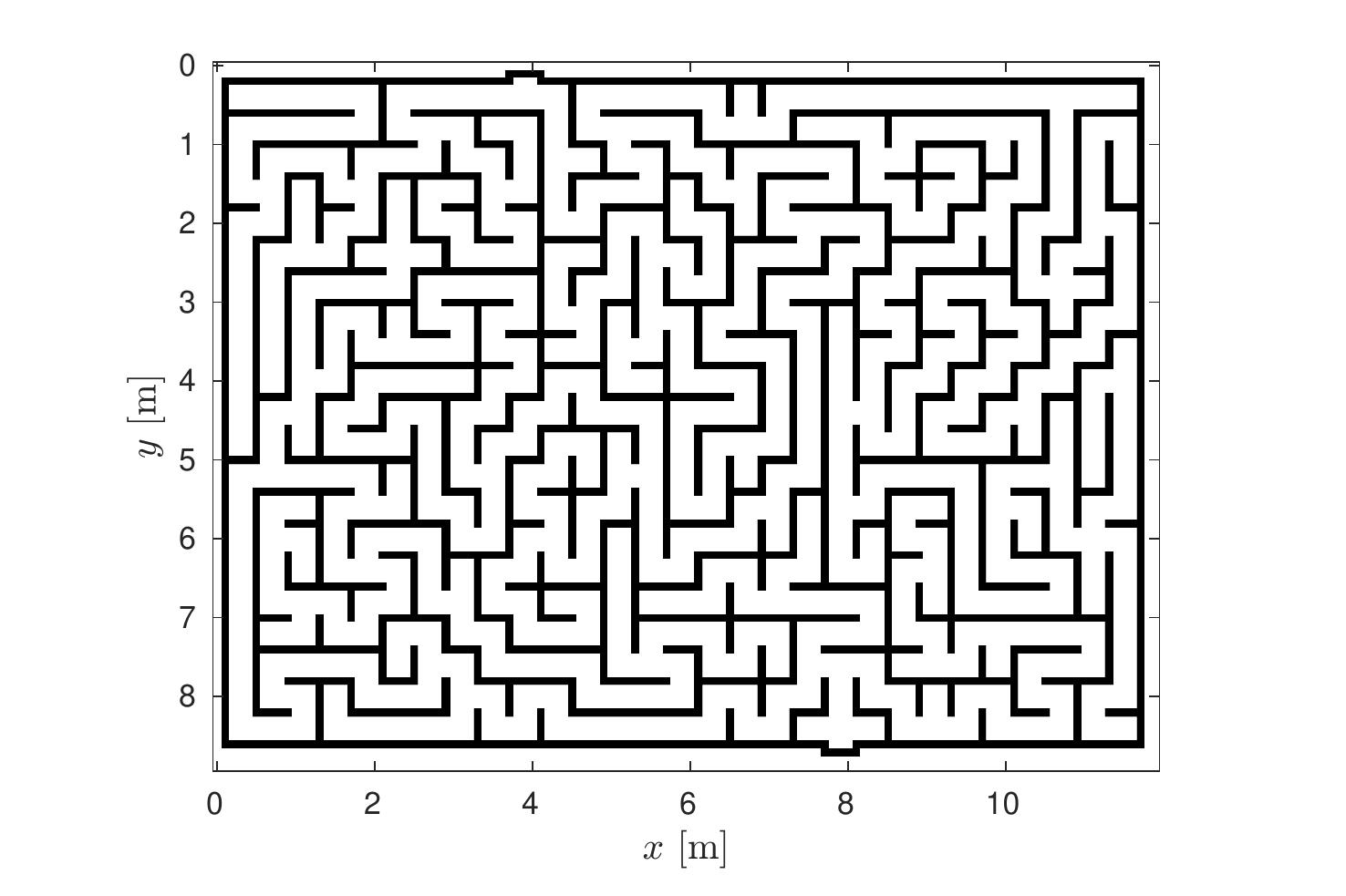}
        \caption{Maze}
        \label{fig:maze}
    \end{subfigure}
    
    \begin{subfigure}[t]{\columnwidth}
        \includegraphics[width=\textwidth]{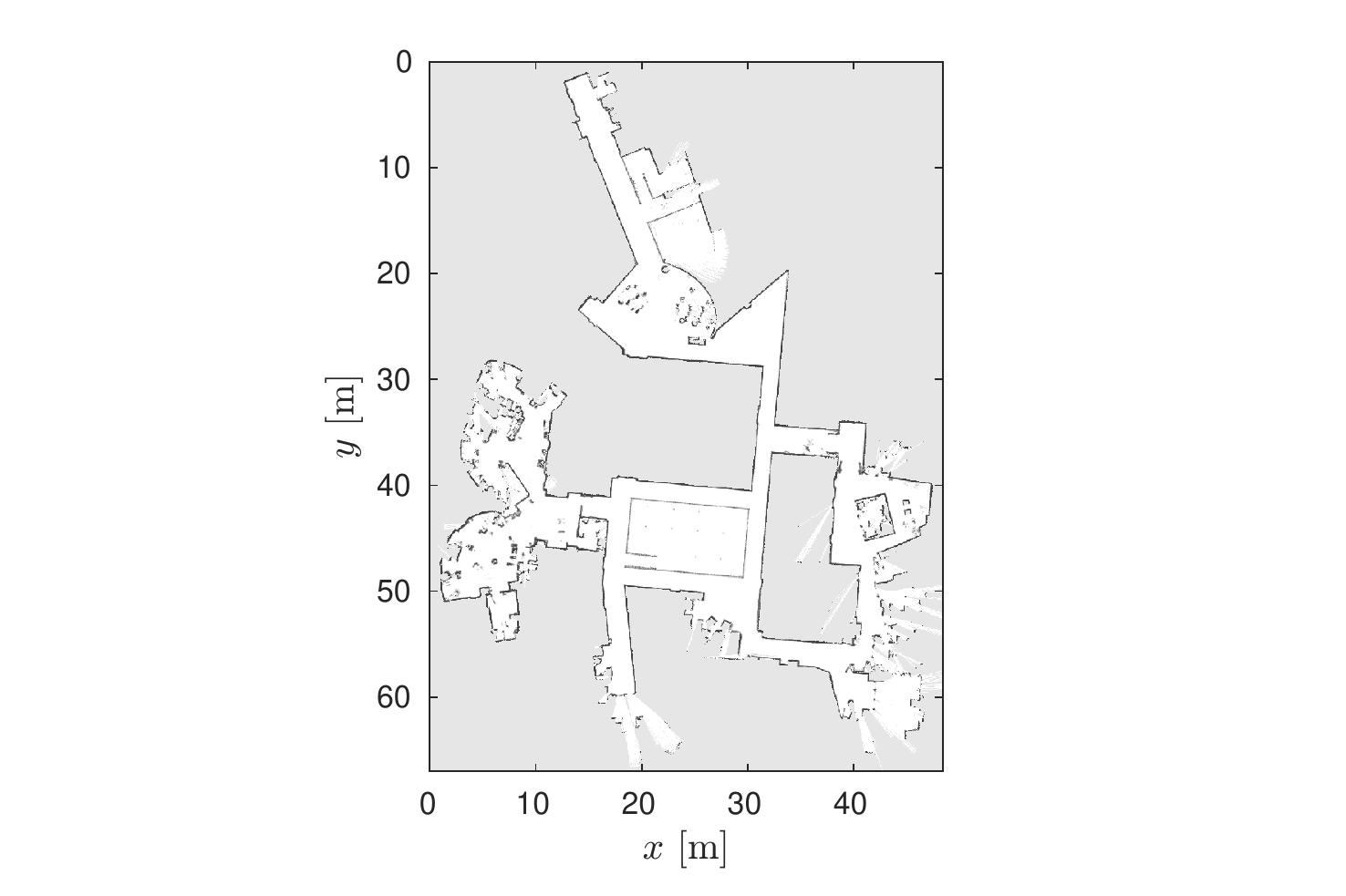}
        \caption{Office}
        \label{fig:mit}
    \end{subfigure}
    
    \begin{subfigure}[t]{\columnwidth}
        \includegraphics[width=\textwidth]{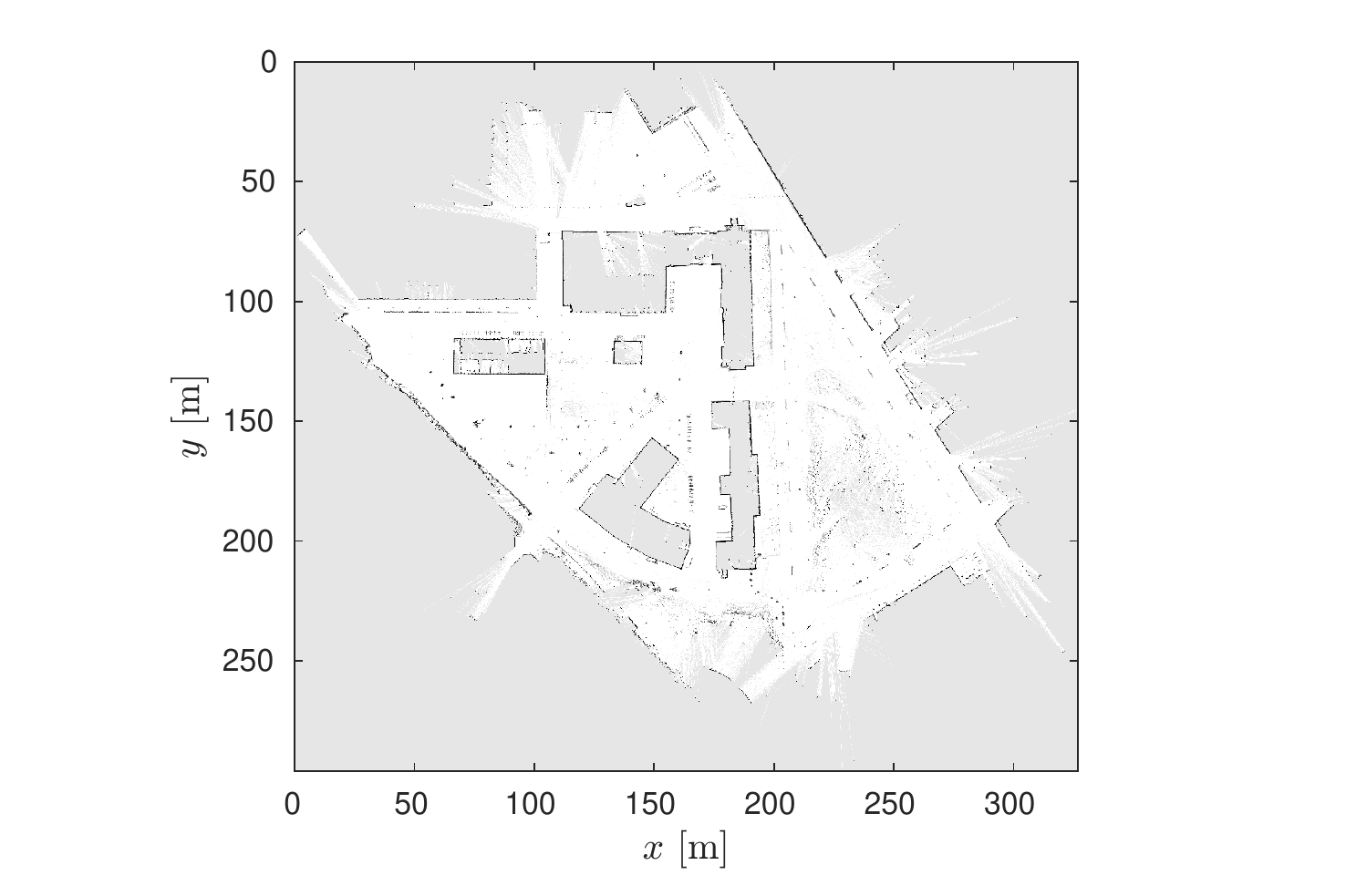}
        \caption{Outdoor}
        \label{fig:fr}
    \end{subfigure}
\caption{Environments examined in the simulation experiments. }
\label{fig:environments}
\end{figure}

The simulated robot is able to turn in place, and it is controlled by applying a linear and angular velocity.
For the maze and office environments, the robot was assigned a maximum linear velocity of 1 meter per second, and for the outdoor environment 3 meters per second.
The angular velocity was constrained to be between -0.5 and 0.5 radians per second.
The robot was equipped with a laser range finder with a maximum range of 4 meters (maze and office environments), or 15 meters (outdoor environment).
Time was discretized into 1 second intervals per decision epoch, and robot trajectories were simulated applying the velocity motion model from~\cite[Ch.~5.3]{Thrun2006}.

For planning with POMCP (Algorithm~\ref{alg:mcts}), the control space was discretized uniformly to 9 linear velocity and 7 angular velocity values, resulting in a total of 63 control actions.
The exploration bonus $e$ was set to 50 for the maze and office environments and 100 for the outdoor environment.
A total of $N_s=3000$ simulation episodes were executed when planning each action.

For planning with SMC (Algorithm~\ref{alg:smc}), the control space does not need to be discretized, but the kernels for sampling control signals must be specified.
It is more likely that the robot obtains more information the greater the amount of unexplored cells covered by its sensors. 
Thus, to increase the likelihood of observing more unexplored cells, the initial control sequences for iteration $l=1$ were sampled assigning higher probability to linear velocities near the robot's maximum velocity.
Angular velocities were sampled uniformly at random. 
The number of particles was $M = 100$, and a threshold value of $M_t = M/4$ was set to trigger resampling.
For iterations $l>1$, we applied a Gaussian kernel with variance proportional to $(1 / l^2)$ times the full range of possible linear and angular velocity values.
Control actions violating the aforementioned maximum and minimum values for the velocities were rejected.
A cooling schedule $\nu_l = 2l+5$ was applied, with $l_{\text{max}} = 7$.

During the simulations, control actions were checked by examining their corresponding trajectories.
If the trajectory entered cells with occupancy probability greater than 0.2, the control action was rejected.
The robot was also not allowed to move outside the map area or enter unknown areas (grey cells in Figure~\ref{fig:environments}).

For prior information, a non-informative and an informative prior were considered.
A non-informative prior corresponds to the case where the map area is initially unknown to the robot, with the exception of information provided by a single observation recorded at the robot's starting location.
The informative prior was designed such that in the maze environment it corresponds to telling the robot the locations of all the maze walls (as indicated by the black cells in Figure~\ref{fig:maze}) by assigning them occupancy probability of 0.99, and in the office and outdoor environments indicating the unknown area (the gray cells in Figures~\ref{fig:mit} and~\ref{fig:fr}) by again assigning them an occupancy probability of 0.99.
For each type of prior and each environment, the experiment was repeated five times.

\subsubsection{Results}
\label{subsubsec:results}

\begin{figure}[t!p]
\centering
	\begin{subfigure}[t]{\columnwidth}
        \includegraphics[width=\textwidth]{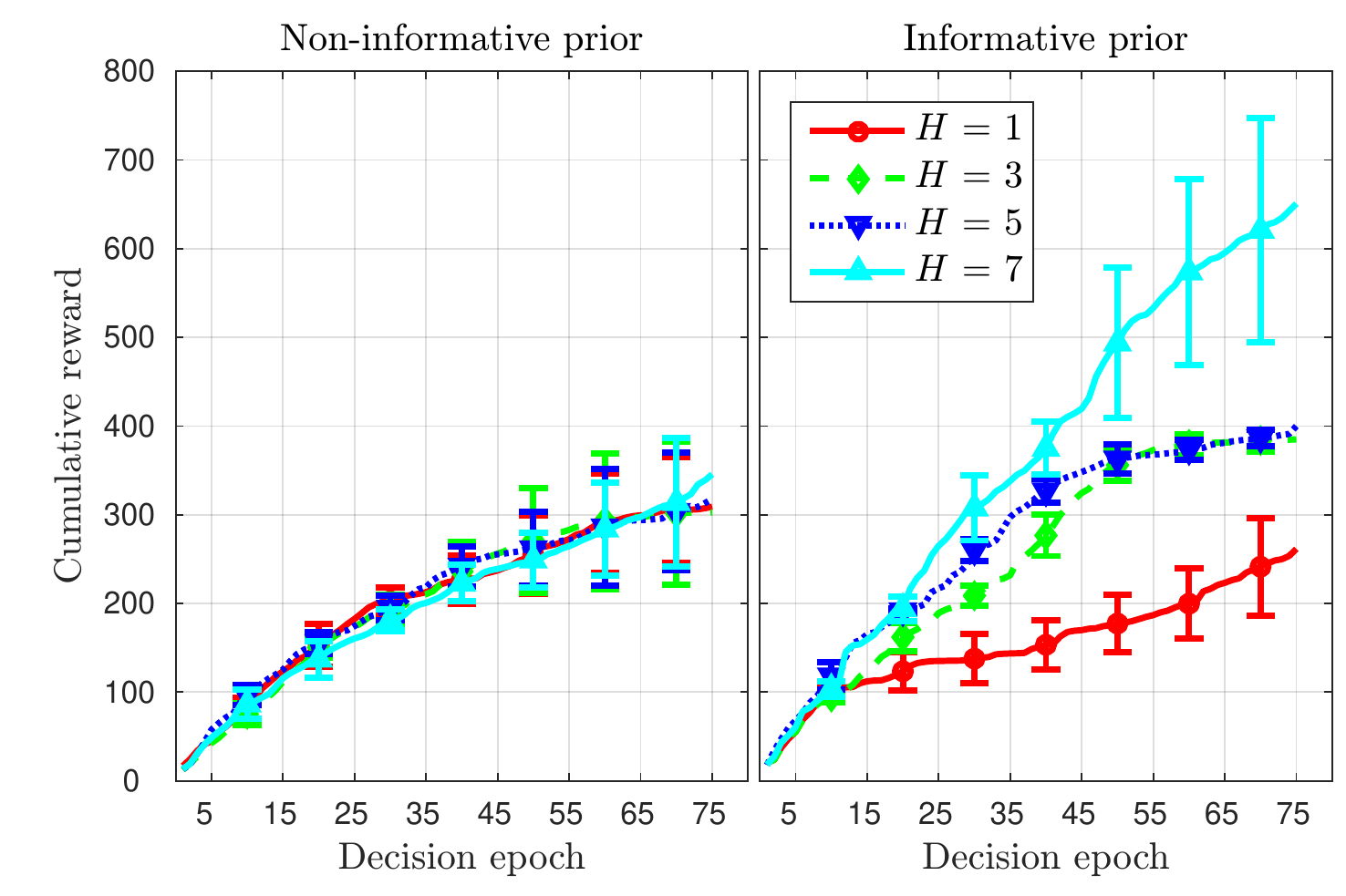}
        \caption{Maze}
        \label{fig:rmaze}
    \end{subfigure}
    
    \begin{subfigure}[t]{\columnwidth}
        \includegraphics[width=\textwidth]{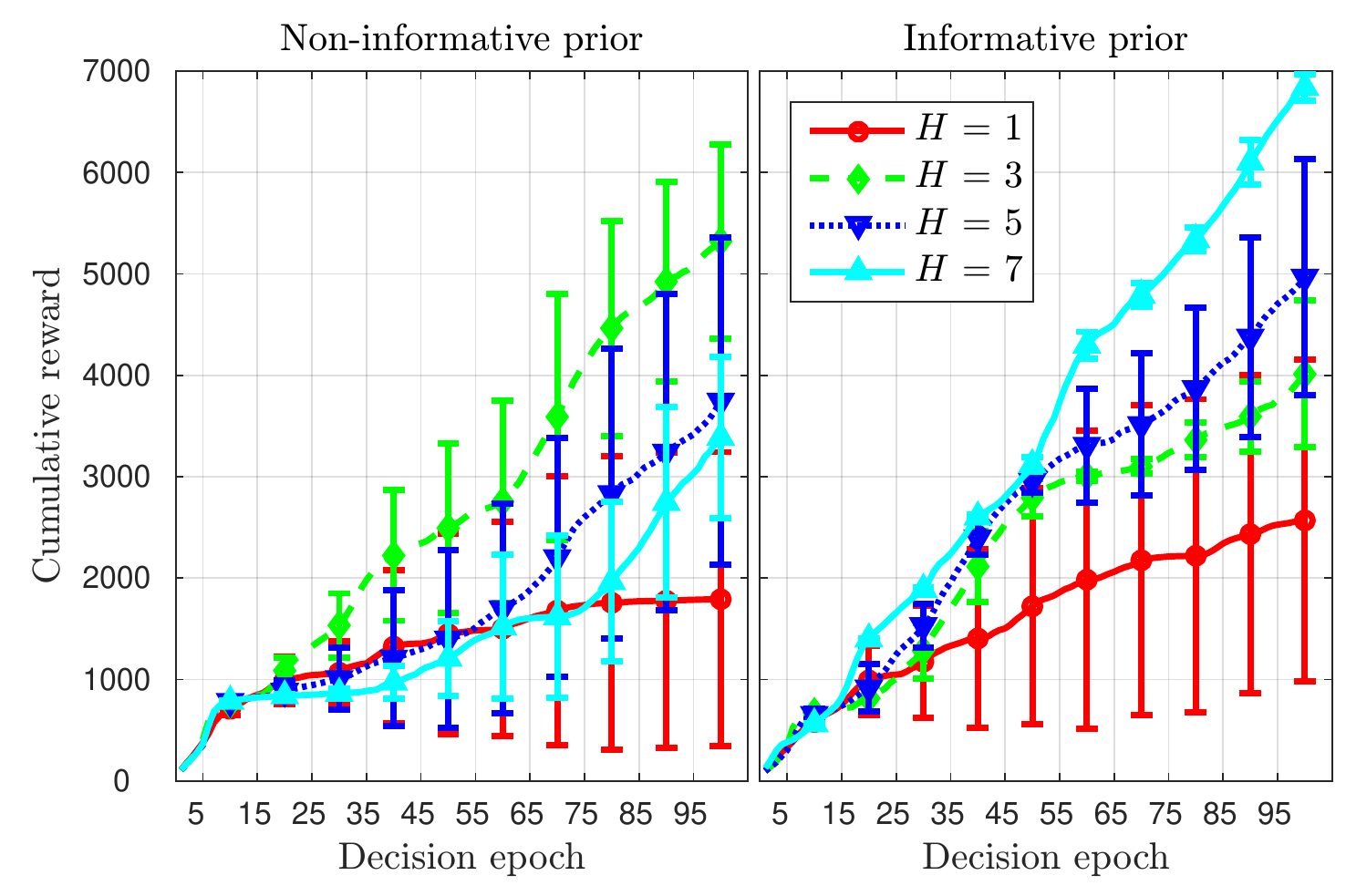}
        \caption{Office}
        \label{fig:rmit}
    \end{subfigure}
    
    \begin{subfigure}[t]{\columnwidth}
        \includegraphics[width=\textwidth]{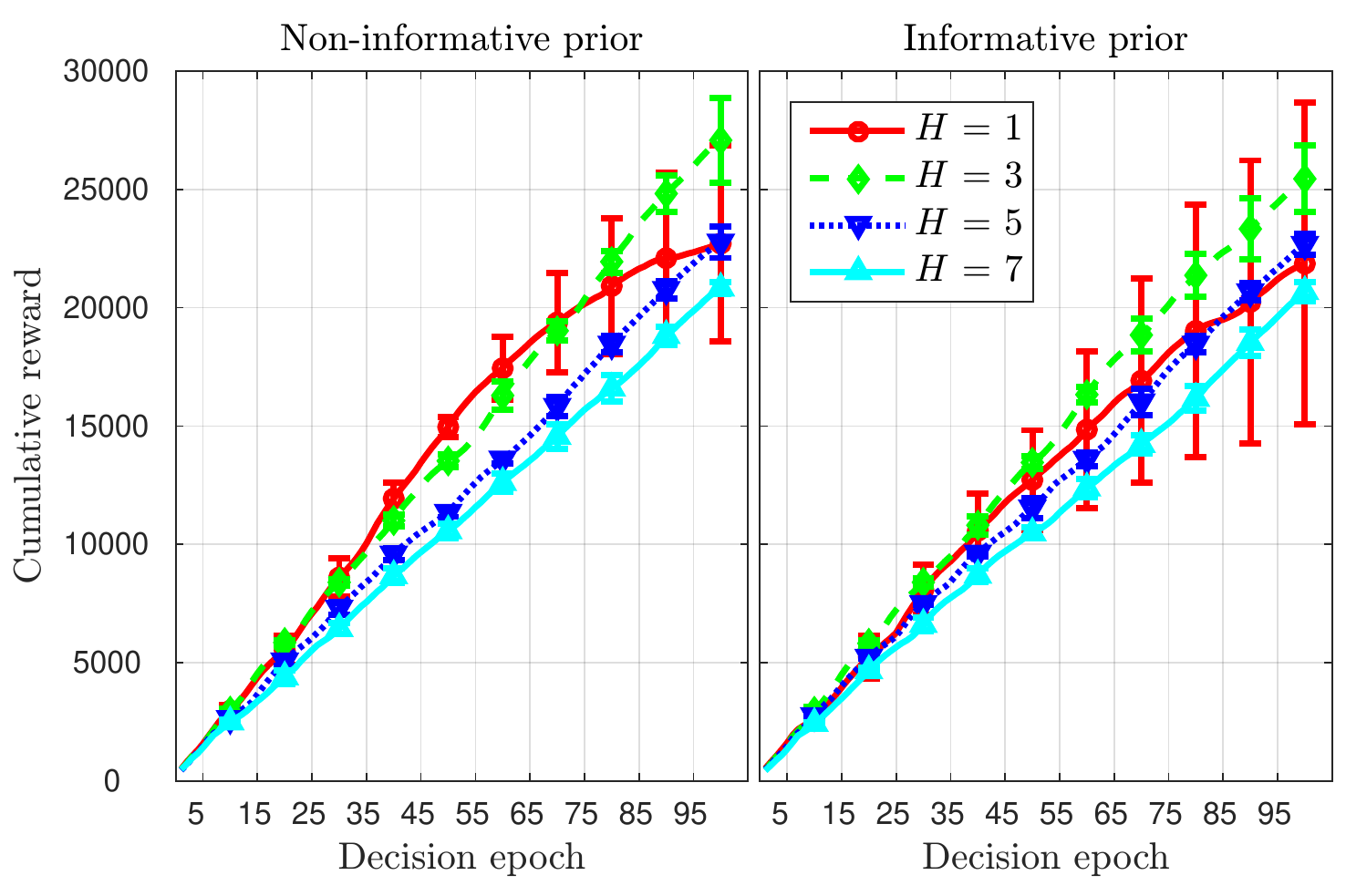}
        \caption{Outdoor}
        \label{fig:rfr}
    \end{subfigure}
\caption{Cumulative reward in the maze, office, and outdoor environments as a function of optimization horizon $H$ applying POMCP (Algorithm~\ref{alg:mcts}). The lines with the markers show the means over 5 simulation runs, while the horizontal bars indicate the 95\% confidence intervals. In each subfigure, the left panels show results for a non-informative prior and right panels for an informative prior.}
\label{fig:rewards}
\end{figure}

\begin{figure}
\centering
\includegraphics[width=\columnwidth]{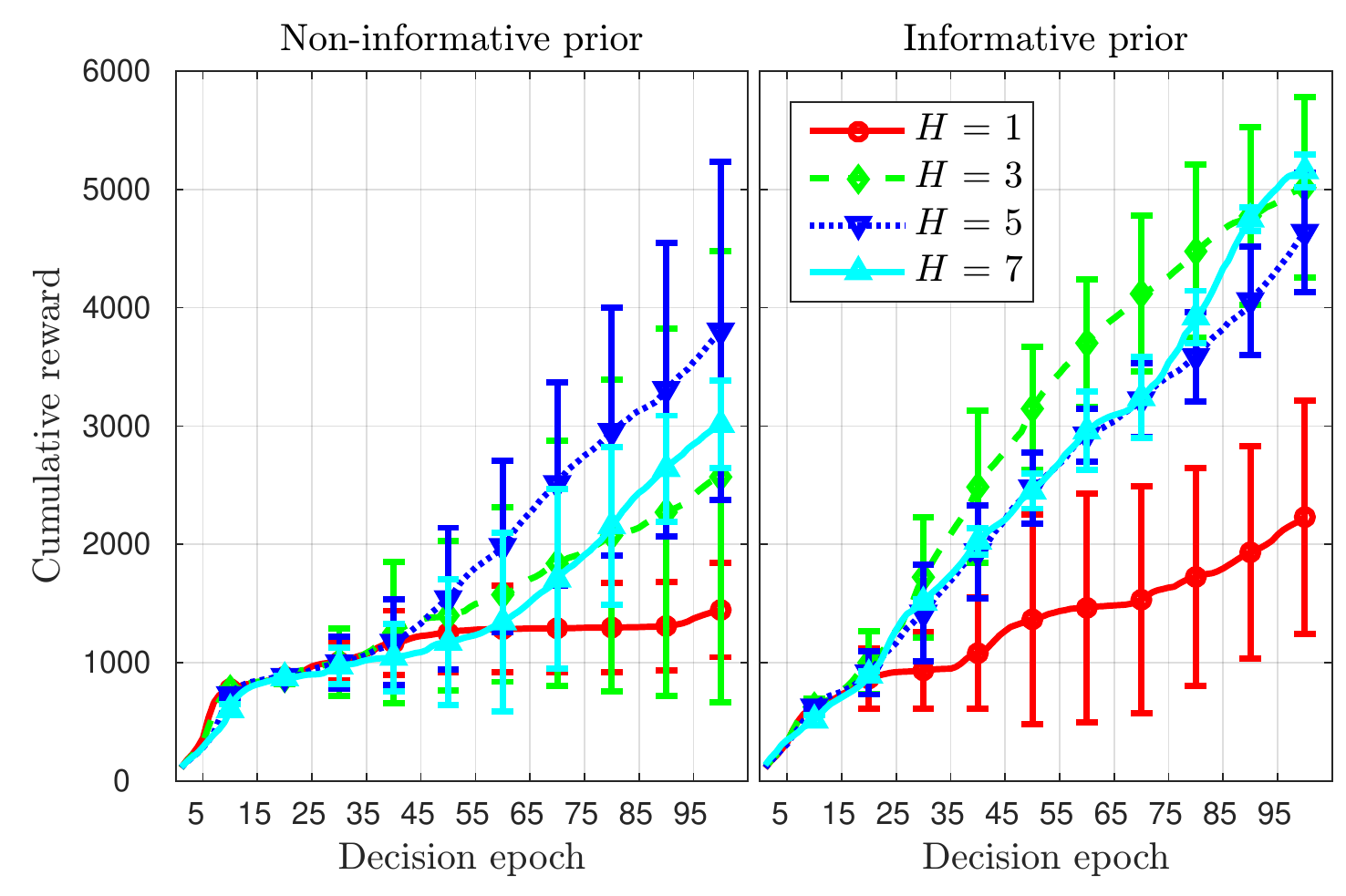}
\caption{Cumulative reward in the office environment applying SMC (Algorithm~\ref{alg:smc}). The lines with the markers show the means over 5 simulation runs, while the horizontal bars indicate the 95\% confidence intervals. The left panel shows results for a non-informative prior, and the left panel for an informative prior.}
\label{fig:rmit_smc}
\end{figure}

Figure~\ref{fig:rewards} shows comparisons of the cumulative mutual information reward collected as a function of the decision epoch when the POMCP algorithm was applied, in each of the environments studied and for optimization horizons $H=1$, 3, 5, and 7.
The lines indicate the mean values, while the horizontal bars indicate the 95\% confidence intervals of the mean values.
For the non-informative prior, increasing the optimization horizon did not consistently result in an increased cumulative reward in any of the environments.
This is due to the fact that in this case the map samples (and hence the corresponding observation samples) are drawn assuming a uniform prior on the map cells' occupancy, which poorly reflects the true configuration of the environment.

In contrast, for the informative prior, statistically significantly greater cumulative rewards were obtained in the maze and office environments when increasing $H$, as indicated by Figures~\ref{fig:rmaze} and~\ref{fig:rmit}.
A similar effect was however not observed in the outdoor environment (Figure~\ref{fig:rfr}).
In the maze and office environments, there are dead-ends, i.e.\ locations where the robot has to travel through already-explored areas to reach new, unexplored areas.
Choosing to traverse towards a dead-end may seem informative in the short term, but yields poor information gain later.
Dead-ends may be avoided if sufficient prior information is available and can be exploited, i.e.\ the optimization horizon is great enough to indicate the presence of a dead-end.
Turning a corner in the maze may severely limit the trajectory options available at subsequent decision epochs.
In contrast, the outdoor environment is primarily open, and committing to a certain trajectory usually does not impose such limitations.
The experimental results suggest that the effect of prior information on exploration performance is lesser in such an open environment.

Similar results were observed for the SMC algorithm, as shown for the office environment in Figure~\ref{fig:rmit_smc}.
Here, the lines indicate the mean values, while the horizontal bars indicate the 95\% confidence intervals of the mean values.
Comparing the right panel of the figure to the right panel of Figure~\ref{fig:rmit}, we note that for the informative prior the effect of increasing the horizon $H$ for SMC does not seem to be as significant as for POMCP.
From the data we determined that for POMCP the robot decided to move from its initial position at $(38.5,58.0)$ towards the large open area around coordinates $(25,50)$ (see Figure~\ref{fig:mit} for the map with coordinate axes) in 1 out of 5 cases for $H=5$, and in all of the five cases for $H=7$.
In contrast, for SMC, the robot instead moves towards the corridor around coordinates $(45,50)$ in 4 out of 5 cases for $H=5$, and in 3 out of 5 cases for $H=7$.
Moving to the open area results in a much greater information gain over a long sequence of decisions, explaining the reason for the difference in favour of POMCP for $H=7$.
We do not believe this to be an indication of the superiority of the POMCP approach, but rather a byproduct of a statistical evaluation of the algorithms combined with the stochastic nature of the solution algorithms themselves.

Overall, the results of the experiments indicate that non-myopic planning ($H>1$) is useful when the available prior information can be leveraged to find more informative local trajectories.
This is the case in particular for the maze environment with an informative prior (Figure~\ref{fig:rmaze}), and the office environment with an informative prior (Figures~\ref{fig:rmit} and~\ref{fig:rmit_smc}).

The information the robot has about the map affects whether increasing the optimization horizon is useful.
In the maze environment with a non-informative prior, the robot typically has information about the map only in its immediate vicinity as its view is blocked by nearby maze walls.
Thus, most local trajectories, short or long, lead the robot outside this area, and increasing the optimization horizon is not useful if the prior information about these areas is not accurate (Figure~\ref{fig:rmaze}).
The office environment is more open, allowing the robot to observe the map in a larger area around it.
The information gained in this way leads to improvements in exploration performance when increasing the optimization horizon, even in the case of an otherwise non-informative prior.
Once the local trajectories considered are such that they bring the robot outside the area about which it has accurate information, performance does not improve anymore: for instance, in the left panel of Figure~\ref{fig:rmit}, this happens for $H>3$.
One further case where increasing the optimization horizon is not useful if the environment is such that all local trajectories are roughly equally informative.
This is the case in the outdoor environment, which consists primarily of open areas (Figure~\ref{fig:rfr}).

\subsection{Combining POMDP based and frontier based exploration}
\label{subsec:frontier_comparison}
The proposed POMDP approach is capable of handling uncertainty in robot and environment states in a principled manner, allowing quantifiable trade-offs between uncertainty reduction and the cost of control actions.
However, the optimization horizon has to be bounded to maintain computational feasibility. 
Thus only \emph{local} reward information can be considered, leading to susceptibility to local minima.

A frontier-based exploration method, see e.g.\ \cite{Yamauchi1997}, detects frontiers between free and unknown areas in the current map.
One of the frontiers is selected as the exploration target, based on, e.g.\, the distance from the robot's current position to the frontier, or the size of the frontier.
The robot is then commanded to move towards the selected target.
Thus, frontier exploration can exploit \emph{global} knowledge of frontiers towards unexplored areas over the whole map.

POMDP based exploration can fail when local information available within the optimization horizon is not sufficient to find an action with good exploration performance, for example if no unexplored area is reachable within the optimization horizon.
To reduce the effect of these types of failures, we implemented an exploration method that combines POMDP based and frontier based exploration.
The method applies POMDP based planning and executes the actions thus found until it reaches a situation where either 1) all local action sequences are below a given informativeness threshold, as measured by the expected total MI for them, or 2) all valid local action sequences correspond to trajectories with a length less than a given threshold value, indicating a possible dead-end.
When either of the conditions triggers, a frontier exploration method is queried once for a frontier to be assigned as the next target for the robot.
When this frontier is reached, the POMDP planning phase is again resumed.

We will argue that combining POMDP based and frontier based exploration in this way presents a stronger alternative to applying either approach alone.
To support this argument, we executed a series of experiments comparing such an approach to only applying frontier exploration.
The software used in this subsection was implemented in C++, and is available at \url{https://goo.gl/ENGkIf}.

\subsubsection{Experimental setup}
\label{subsubsec:comparison_setup}
We chose to conduct the experiments in the office environment (Figure~\ref{fig:mit}), as based on Subsection~\ref{subsubsec:results} increasing the optimization horizon $H$ there improves performance both in the case of non-informative and informative prior.
We implemented the method combining POMDP based and frontier based exploration applying the SMC algorithm and the basic frontier exploration algorithm as presented e.g.\ in \cite{Yamauchi1997}.
The SMC method was chosen since it dynamically generates feasible control signals and trajectories, without need to manually define a fixed set of primitive control actions from which the trajectories are constructed.
This method was then compared to only applying the basic frontier exploration method.

Each of the exploration experiments was repeated five times.
Each repetition was terminated either at a timeout of 400 seconds, or if a failure happened that either caused the robot to get stuck, or the planner software to fail to produce a result.
As in Subsection~\ref{subsubsec:sim_setup}, the initial information provided to the robot consisted of a single observation recorded at the starting location, in addition to possible prior information.

For the SMC algorithm, we set the number of particles to $M = 20$, and the resampling threshold to $M_t = M/4$.
A cooling schedule $\nu_l = 2l+5$ was applied, with $l_{\text{max}} = 4$.
The simulated robot and the kernels applied for sampling control actions were as described in Subsection~\ref{subsubsec:sim_setup}.
We applied maps with a resolution of 0.05 meters per cell, and set the threshold for triggering frontier exploration at a total trajectory length of 0.5 meters or expected MI of less than 50 bits.
The 50 bit MI threshold corresponds to 50 completely unknown cells (occupancy probability 0.5) becoming completely known (occupancy probability 0 or 1), so, roughly speaking, if the robot expected to explore less than 0.125 square meters of new area, or move less than 0.5 meters, it would trigger the frontier exploration method once instead of continuing with POMDP based exploration.

A conceptual overview of our software implementation is shown in Figure~\ref{fig:implementation}. 
The robot is interacting with the environment, shown on the bottom left hand side of the figure.
The outputs from the environment, i.e., observations $z_t$, are processed by the SLAM algorithm to revise the belief state $b_t=p(x_t,m_t)$.
Based on the current belief state, the planner module shown on the top right of the figure computes an optimized sequence of control actions $\hat{u}_{t:t+H-1}$.
A controller module shown on the bottom right of the figure decides which control action $u_t$ to finally apply.

\begin{figure}
\centering
\includegraphics[width=\columnwidth]{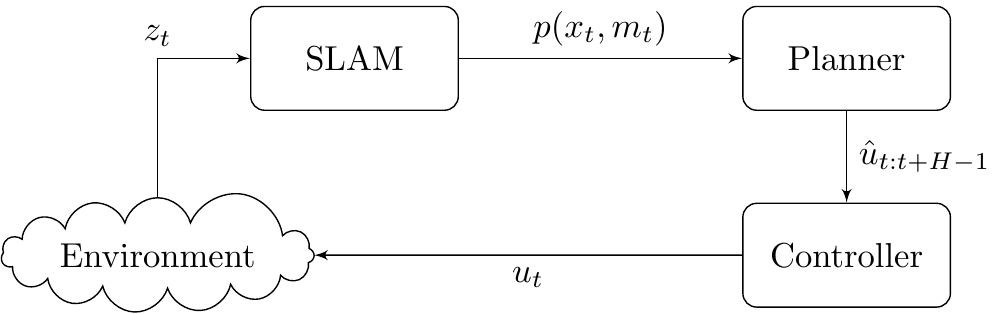}
\caption{A software implementation of the planner. The cloud-shaped block depicts the environment the robot is interacting with. The rectangular blocks indicate software modules, and arrows indicate propagation of signals between the modules, labelled by the mathematical symbol of the signal.}
\label{fig:implementation}
\end{figure}

The simulation was implemented in the Stage robot simulator~\citep{Vaughan2008}, integrated with the Robot Operating System (ROS) framework~\citep{Quigley2009} where the exploration algorithms were implemented.
The SLAM module we applied (see Figure~\ref{fig:implementation}) was the RBPF SLAM algorithm based on~\cite{Grisetti2007}.
As the frontier exploration method, we applied an open source implementation provided at \url{http://wiki.ros.org/frontier_exploration}.

We remark that there exist multi-robot frontier-exploration techniques that make use of prior information, e.g., in the form of semantic information about types of areas in the map~\citep{Stachniss2009,Quattrini2015}.
Suitable applications for these methods include for example search and rescue, where semantic information about the types of rooms, e.g., office or lobby, can help guide the robots to promising search areas.
However, in the experiments here, no such semantic information was available.

\subsubsection{Results}
\label{subsubsec:comparison_results}
Overall, in the 40 exploration runs applying the proposed method (5 runs for each of the 4 horizons, with 2 cases for prior information), we observed 10 failures with the experiment terminating before the timeout of 400 seconds.
There were 8 cases of planner failure, either due to being unable to detect a frontier after trying POMDP based exploration, or due to inability to find a feasible path towards the requested exploration target. 
There were 2 cases in which the robot got stuck in the simulator after a collision.
Among all failure cases, the earliest time of occurrence was at 195 seconds, while the average time of failure occurrence was at 273 seconds.
In the 5 runs with the pure frontier exploration method, there was 1 failure at 197 seconds due to inability to find any frontiers.

\begin{figure}[t!p]
\centering
	\begin{subfigure}[t]{\columnwidth}
        \includegraphics[width=\textwidth]{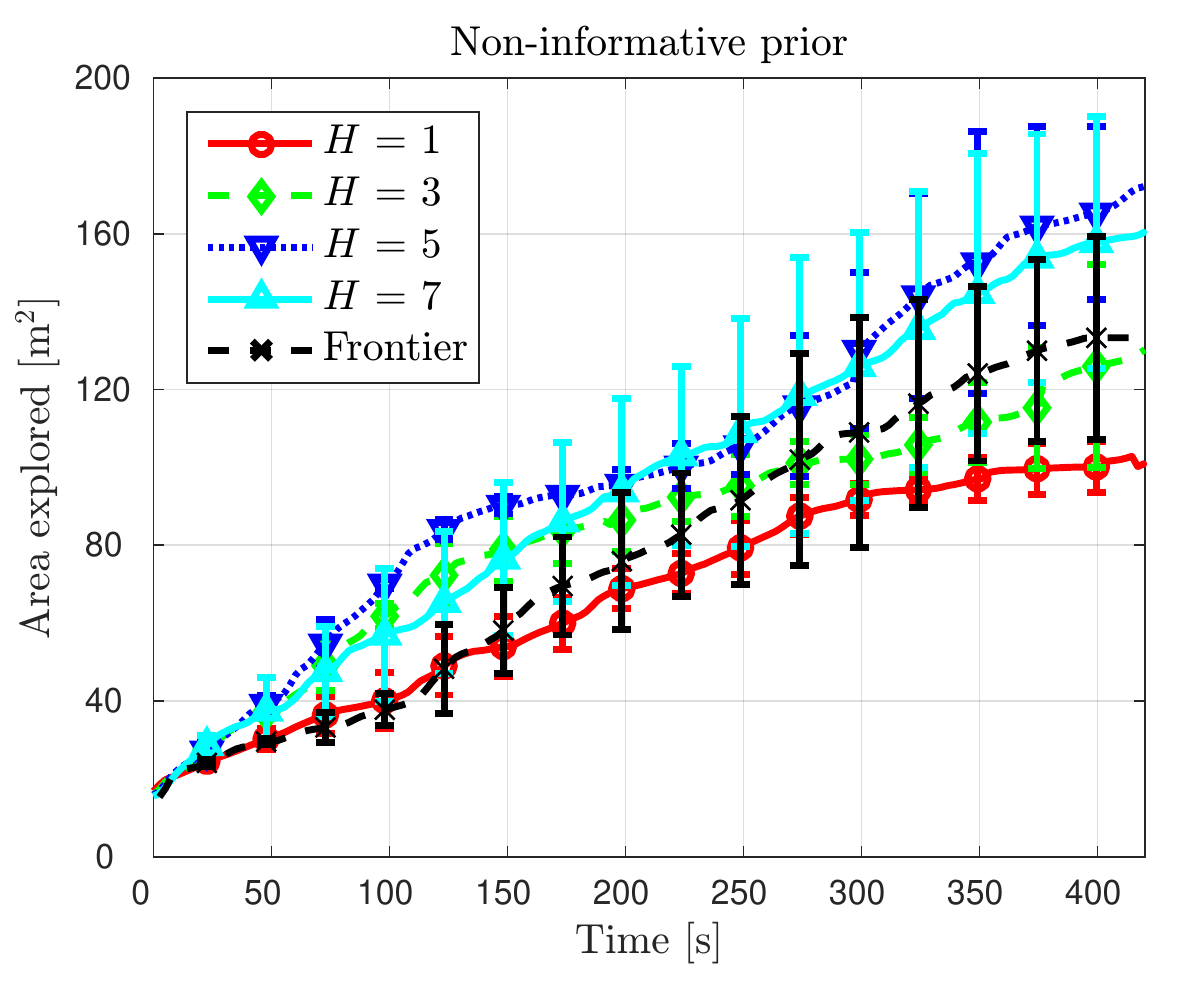}
        \caption{Non-informative prior}
        \label{fig:rcomparison_noninformative}
    \end{subfigure}
    
    \begin{subfigure}[t]{\columnwidth}
        \includegraphics[width=\textwidth]{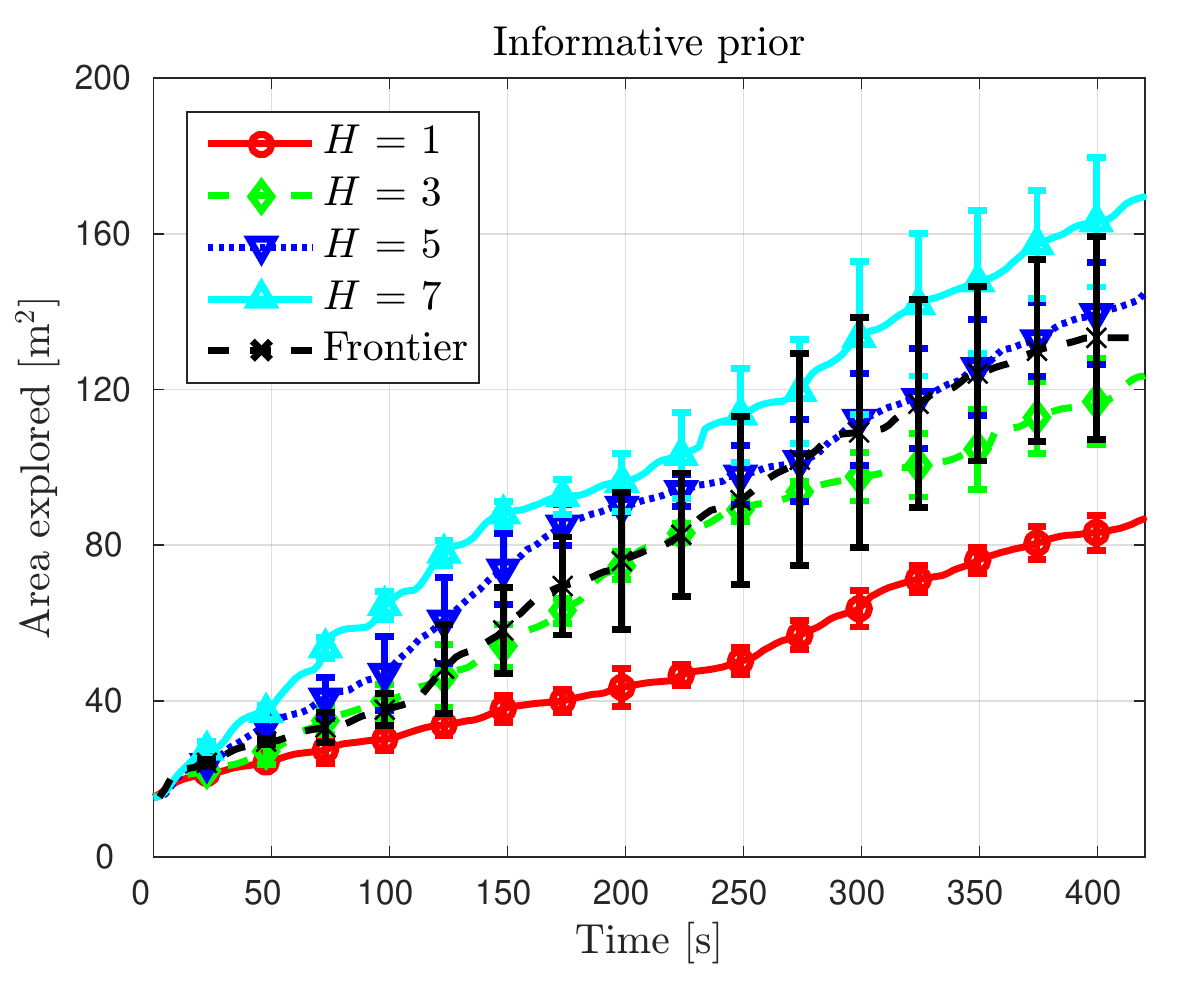}
        \caption{Informative prior}
        \label{fig:rcomparison_informative}
    \end{subfigure}
\caption{Comparison of POMDP based exploration (with optimization horizon $H$) combined with frontier based exploration with pure frontier exploration in the office environment. The lines with the markers show the mean area explored over 5 simulation runs, while the horizontal bars indicate the 95\% confidence intervals. Figure~\ref{fig:rcomparison_noninformative} shows the results for a non-informative prior, and Figure~\ref{fig:rcomparison_informative} for an informative prior.}
\label{fig:rcomparison}
\end{figure}

Since frontier exploration does not consider mutual information in selecting exploration targets, we present the comparison results in terms of the total area explored as a function of the time spent exploring.
Figure~\ref{fig:rcomparison} shows the mean area explored and its 95\% confidence interval over each of the five experiments, for each of the methods applied.
The results shown take into account that not every experiment ran until the timeout of 400 seconds.

There seems not to be a significant difference in the area explored between the cases of non-informative or informative prior.
However, we can see that the consistency of the results is better in case of the informative prior: the mean area explored increases monotonically as a function of $H$, and the confidence intervals are smaller than for the non-informative prior.

We note that with the exception of the myopic case $H=1$, the proposed method does not perform significantly worse than pure frontier exploration during any time interval.
It is interesting to note that although the area explored at the end of the experiment is not significantly greater for the proposed method than for pure frontier exploration, there are time intervals where the difference is significant in favour of the proposed method.
Such intervals can be seen for example from 75 seconds to about 200 seconds for the non-informative prior, $H=5$, and from 75 seconds to about 175 seconds for the informative prior, $H=7$.

When we examined the trajectories the robot chose in each of these cases more closely, we discovered that this difference is primarily due to a different choice of initial exploration target.
Recall that the robot started at coordinates $(38.5,58.0)$ (see Figure~\ref{fig:mit}).
The proposed method, especially with an informative prior, favours moving towards the corridor at coordinates $(40, 50)$, as moving instead to the room at coordinates $(40,60)$ is noted to quickly lead to a dead-end providing no further information gain.
This is however not considered by frontier exploration, which favoured moving first to the frontier in the room, and then returning to explore other areas once the dead-end was discovered.

A further example of the usefulness of prior information can be seen in Figure~\ref{fig:smc_convergence}.
The prior information indicates the outer edges of the environment, shown in the figure by the black lines bordering the unknown area.
Initially, the trajectories sampled are distributed evenly leading towards the corridor on the left hand side of Figure~\ref{fig:smc_convergence_a} and the area on the right hand side of the figure.
Availability of prior information however indicates that a smaller unknown area will be visible when the robot moves to the left hand side corridor, compared to moving to the potentially large open area on the right hand side.
As the total MI related to either trajectory choice is evaluated, the SMC algorithm eventually discards trajectories towards the corridor and converges on a trajectory bringing the robot towards the open area instead, as seen in Figure~\ref{fig:smc_convergence_c}.
In cases with a non-informative prior, both trajectory options result in roughly equal expected MI.

Since the proposed method combines POMDP based and frontier based exploration, it is illustrative to consider how often either of the exploration techniques is applied.
Table~\ref{tab:callnumbers} indicates the average number of times the proposed method applied either POMDP based or frontier based exploration to select the next target where the robot should move to explore the environment.
In the majority of the cases, POMDP based exploration is preferred.
As expected, shorter optimization horizons $H$ lead the robot more frequently to situations where no informative local trajectories can be found, and subsequently frontier exploration is applied more frequently.
We also note that applying the informative prior seems to result in less calls to frontier exploration, indicating further the usefulness of prior information for POMDP based exploration.
Overall, there is a slightly decreasing trend in the total number of calls to either method as a function of $H$, since the robot tends to traverse a longer trajectory before considering the next exploration target for longer optimization horizons.

\begin{table}[ht]
\centering
\caption{The average number of function calls in the proposed method over five experiments to either the POMDP based or the frontier based exploration method. Results are shown as a function of the optimization horizons $H$, and for both the case of non-informative and informative prior information.}
\label{tab:callnumbers}
\begin{tabular}{@{}cccc@{}}
\toprule
                                                 &$H$& POMDP & Frontier \\ \midrule
\multirow{4}{*}{Non-informative prior}           & 1 &     28.2               &   10.0       \\
                                                 & 3 &     22.2               &    5.0      \\
                                                 & 5 &     20.0               &    3.0      \\
                                                 & 7 &     23.2               &    3.4      \\ \midrule
\multirow{4}{*}{Informative prior}               & 1 &     25.2               &    5.0      \\
                                                 & 3 &     25.2               &    3.4      \\
                                                 & 5 &     19.6               &    2.4      \\
                                                 & 7 &     18.6               &    1.8      \\ \bottomrule
                                                 
\end{tabular}
\end{table}

Based on the experimental results, our proposed method can outperform frontier based exploration in terms of area explored in the initial part of exploration when the environment is still largely unknown.
We note this is e.g., due to the ability of the proposed method to avoid dead-ends.
Complete exploration of an environment is the goal in many applications, meaning that also dead-ends should eventually be explored.
We note that although the proposed method combining POMDP based and frontier based exploration will eventually achieve this, the strategy it applies may not be optimal: considering the objective of complete exploration, it might be worthwhile to explore nearby dead-ends immediately rather than returning to them later after exploring other parts of the environment.
When this strategy can be improved upon depends at least on the availability of prior information, and remains to be studied more carefully.
In conclusion, we believe the proposed method is preferable in applications where quickly exploring the local environment for a high information gain is required, and completeness of exploration in the short term is not crucial.\section{Real-world exploration}
\label{sec:experiment}
To verify the feasibility of applying our proposed planning approach in a real application, in this section we present results on exploration tasks in a real environment.
The experimental setup is described in Subsection~\ref{subsec:setup}, and results are reported in Subsection~\ref{subsec:results}.

\subsection{Experimental setup}
\label{subsec:setup}
The experiments were executed at a university campus library.
A partial map of the environment with the robot's starting location indicated is shown in Figure~\ref{fig:library_map}.
The environment features open areas and narrow corridors between bookshelves.
In all experiments, the robot started at the location indicated by the black circle marker in the figure.
A view of the environment showing the robot at its starting location is shown in Figure~\ref{fig:library_view}.
The photograph was taken such that it shows the environment towards the negative values of the $x$-axis in the map of Figure~\ref{fig:library_map}.
The robot had no prior information about the environment beyond a single observation recorded at the starting location.

\begin{figure}
\centering
		\begin{subfigure}[t]{\columnwidth}
        \includegraphics[width=\textwidth]{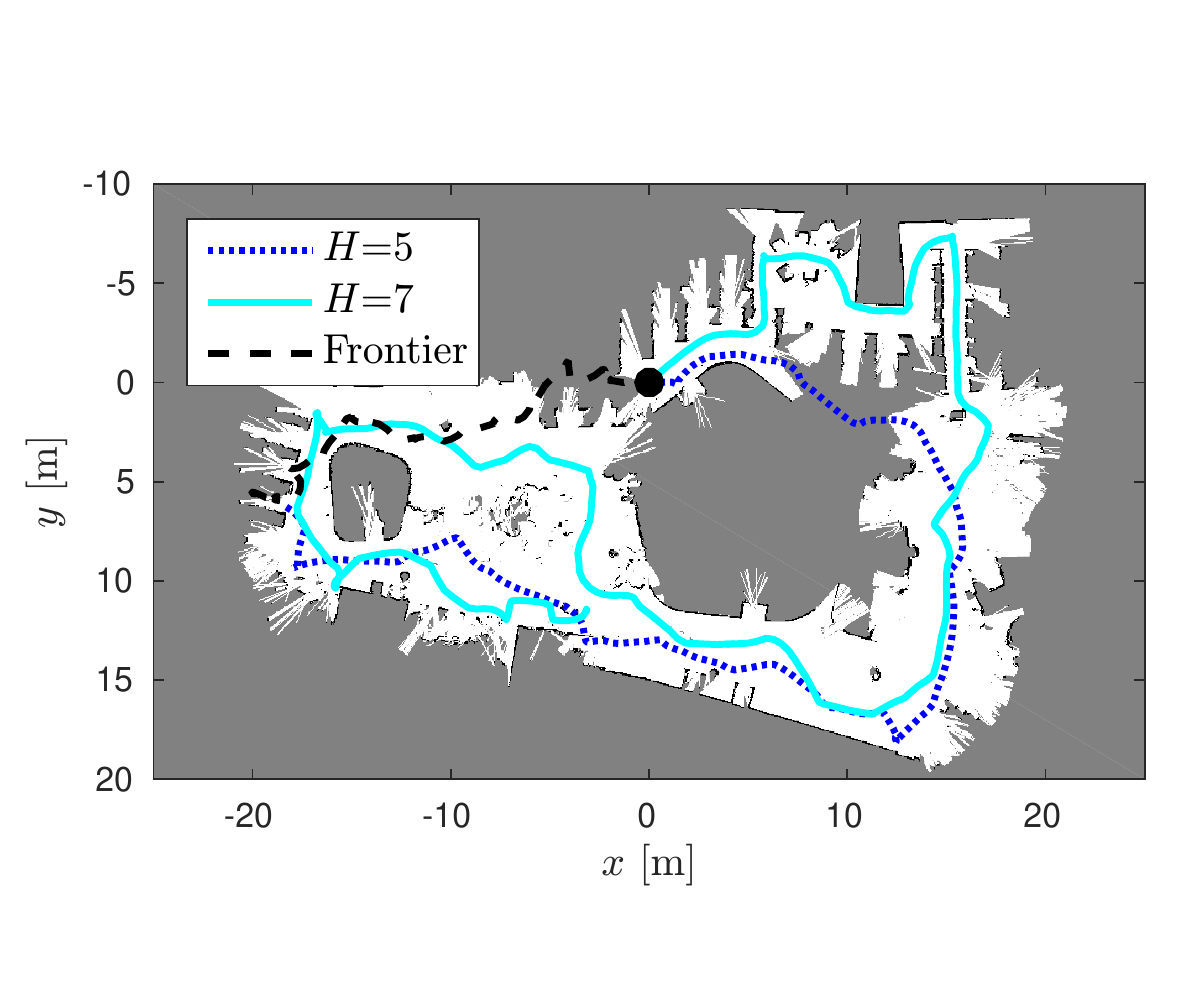}
        \caption{}
        \label{fig:library_map}
    \end{subfigure}
    
    \begin{subfigure}[t]{\columnwidth}
        \includegraphics[width=\textwidth]{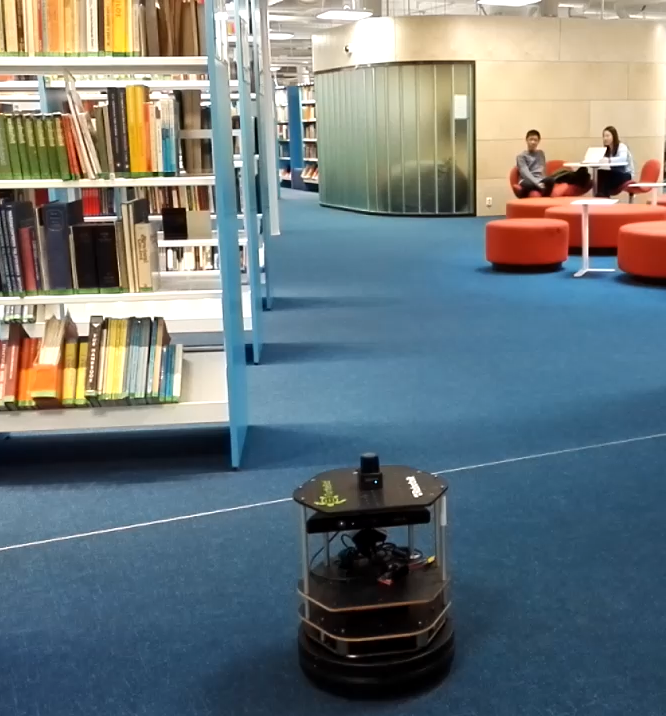}
        \caption{}
        \label{fig:library_view}
    \end{subfigure}
\caption{(a): A partial map of the campus library environment. Occupied, free, and unknown areas are indicated by black, white, and gray cells, respectively.  The robot's starting location is indicated by a filled black circle marker. Overlaid on the map are examples of typical trajectories taken applying the proposed exploration method or frontier based exploration. The map shown corresponds to the data collected while traversing the trajectory for $H=7$. The unknown areas through which the other trajectories travel were observed to be free in the corresponding experiments, but not shown here. (b): A view of the environment from the robot's starting location.}
\label{fig:library}
\end{figure}

The experiments were carried out with a robot such as described in Subsection~\ref{subsubsec:sim_setup}.
The robot's velocities were equal to those of the simulated robot of Section~\ref{subsubsec:sim_setup}.
The robot was equipped with a laser range finder with a maximum range of 4 meters.
The robot can also be seen in Figure~\ref{fig:library_view}.
The software setup and algorithm parameters were as described in Subsection~\ref{subsubsec:comparison_setup}.
The robot's on-board computer, equipped with an Intel i7-4500U multicore processor, 4GB of RAM and running a Linux operating system, was applied to run all of the required software.
We applied the proposed exploration method combining POMDP based and frontier based exploration, with the software implementation described in Subsection~\ref{subsec:frontier_comparison}, and a frontier based exploration method.
Optimization horizons $H=5$ and $H=7$ were applied in the POMDP based exploration method.
With each method, the experiment was repeated 5 times, and each experiment ran until a timeout of 400 seconds or until a failure caused termination of the experiment.

\subsection{Results}
\label{subsec:results}
For the proposed method, we observed 4 failures with $H=5$ occurring at 240, 255, 265, and 376 seconds, and 1 failure with $H=7$ at 390 seconds.
With frontier exploration, we observed 2 failures at 355 and 380 seconds.
The failures were due to the robot colliding with an undetected obstacle and getting stuck, or the planner failing to find an exploration target.
Failures to find an exploration target most often happened at a narrow corridor, where the robot's sensors erroneously indicated that there was no possible path without a collision to exit the corridor.

The minimum, average and maximum planning times required for the SMC algorithm are reported as a function of the optimization horizon $H$ in Table~\ref{tab:plan_times}.
Longer trajectories and larger open areas tend to increase the amount of map cells for which an occupancy value must be sampled to obtain an estimate of the MI, thus increasing the required planning time.
The shortest planning times were observed when the robot was near a wall or in a narrow corridor.
During the planning time, the robot remained in place.

\begin{table}[]
\centering
\caption{The minimum, average and maximum planning times for the SMC algorithm as a function of the optimization horizon $H$.}
\label{tab:plan_times}
\begin{tabular}{@{}cccc@{}}
\toprule
\multirow{2}{*}{$H$} & \multicolumn{3}{c}{Planning time {[}s{]}} \\ \cmidrule(l){2-4} 
                     & Min          & Avg          & Max         \\ \midrule
5                    &  2.60            & 4.24             & 6.14            \\
7                    &  2.03            & 5.21             & 7.61            \\ \bottomrule
\end{tabular}
\end{table}

A summary of the experimental results is presented in Figure~\ref{fig:library_results}, showing the mean area explored and its 95\% confidence interval as a function of time, for each of the exploration methods considered.
For the proposed method with $H=5$, results are only shown up to around 265 seconds, as only two of the experiments ran for a longer time.
In all cases, the confidence intervals drawn take into account the number of active experiments remaining.
The results suggest that in this environment, it is beneficial to apply the proposed exploration approach instead of frontier exploration.
We observed the proposed approach with $H=7$ to result in a significantly greater total area explored than frontier exploration after 200 seconds.

To understand why this is the case, it is useful to study the trajectories traversed by the robot applying each of the exploration methods.
Figure~\ref{fig:library_map} shows typical examples of trajectories for each of the methods.
Frontier exploration, indicated by the black dashed line, consistently chose to move towards the negative $x$ axis from the starting location.
In this direction, there was a row of bookshelves aligned as shown in Figure~\ref{fig:library_view}.
As indicated by the curves in the initial part of the trajectory, frontier exploration would find the nearest frontier behind the corner of the closest bookshelf to the robot.
The robot would then move to it, and repeat in a similar manner until reaching the end of the row of shelves.
Frontier exploration does not attempt to quantify the amount of information that may be gained by such a strategy, but rather deterministically moves the robot towards the closest frontier found.

In contrast, we observed the proposed method to adopt a different strategy.
Since the narrow corridors between bookshelves can often be observed before it is necessary to move into them, this information may be applied to select more informative trajectories.
The robot would generally prefer to avoid narrow corridors, as the proximity of the walls made it unlikely that moving there would provide as much information as moving towards a potentially open area.
Corridors were typically only preferred when there was no alternative, such as shown in the trajectory for $H=7$ indicated by the cyan line in the top right part of Figure~\ref{fig:library_map}.
Trajectories with $H=5$ were similar, as indicated by the blue line in Figure~\ref{fig:library_map} and the results in Figure~\ref{fig:library_results}.
Based on similar results in the previous experiments, we thus believe that increasing the optimization horizon further would not result in better performance in this environment.

\begin{figure}
\centering
\includegraphics[width=\columnwidth]{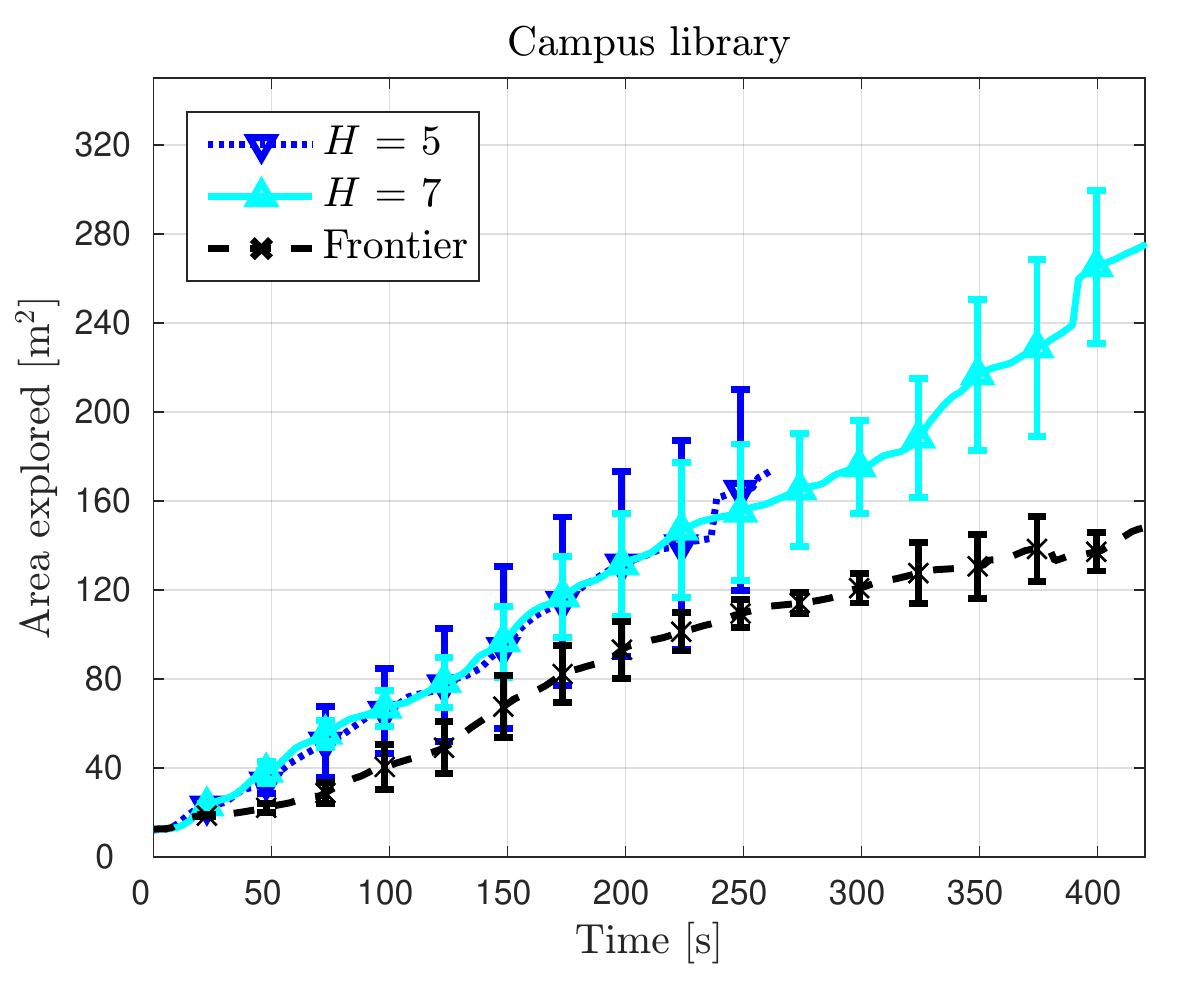}
\caption{Area explored in the campus library environment as a function of time. Results are shown for the proposed method with optimization horizon $H=5$ or $H=7$, and for frontier exploration. The lines with the markers show the mean area explored over 5 runs, while the horizontal bars indicate the 95\% confidence intervals.}
\label{fig:library_results}
\end{figure}

We remark that the results could be different if the robot's starting position were to be changed, as the location to explore first would be different.
For the proposed method, this is due to a different set of feasible local trajectories considered.
For pure frontier exploration, the first location to explore depends on the map information obtained from sensor readings at the starting position.
If the starting position is changed such that the map information remains roughly the same, the same exploration location would likely be selected.

Overall, the experiments show that in the environment studied, combining non-myopic POMDP based planning with frontier exploration improves performance over pure frontier exploration.
Although frontier exploration can apply all global information to select the next exploration target, it performs the selection in a heuristic manner, ignoring  the expected information gain.
POMDP based methods can evaluate the information available locally within the optimization horizon to select exploration targets that lead to quantifiably greater expected information gain.
The weakness of the method, susceptibility to local minima due to the finite horizon, can be alleviated by combining it with traditional frontier exploration.
\section{Conclusion}
\label{sec:conclusion}
We formulated a robotic exploration problem as a partially observable Markov decision process (POMDP), applying mutual information as reward function.
We solved an open loop approximation to the POMDP applying forward simulation and the receding horizon control principle.
We derived a new sampling-based approximation for mutual information, and presented an efficient method to draw samples for the approximation when an occupancy grid map representation is applied.

The usefulness of non-myopic decision-theoretic planning for exploration was demonstrated in simulation and real world experiments.
Non-myopic planning can help avoid situations where initial information gain related to a control action seems high, but ultimately the action leads to a dead-end where further exploration is not possible.
However, basing exploration decisions purely on a finite horizon look-ahead was found to perform poorly in situations where a sequence of control actions longer than the look-ahead horizon is required to reach an unexplored area.
To alleviate this susceptibility to local minima, we suggested combining POMDP based and frontier based exploration.
Experimental results show that, depending on the environment to be explored, this combination can improve exploration performance when compared to only applying frontier exploration.

There are two main weaknesses in our approach.
First, an open loop approximation was required to deal with the very large and possibly continuous action and observation spaces, resulting in a performance loss compared to the optimal closed loop solution.
Secondly, to increase computational effectiveness we ignore the SLAM problem while solving the open loop control problem.
Thus if the possibility of losing the consistency of the real process SLAM filter cannot be ignored~\citep{Blanco2008, Carlone2010}, actions that lead to such a failure state may be selected.

More prior information about the environment was noted to improve exploration performance.
In future work, a database of typical map features for different types of environments, for instance, office, outdoor, etc.\, could be maintained, and map samples drawn from the database instead, see e.g.\ \cite{Strom2015}.
As forward simulation methods are applicable to very general underlying models, more realistic environment models that lift the assumption of spatial independence between occupancy grid cells could be applied, for instance higher order Markov random fields~\citep{Nabbe2007}.

\section*{Acknowledgments}
We thank Mr. Joonas Melin and Mr. Eero Hein\"{a}nen for their support in executing the experimental work.
This work was partly supported by the TUT Graduate School and TUT Strategic funding for robotics and intelligent machines.

\bibliographystyle{IEEEtranN}
\bibliography{ref}

\end{document}